\documentclass{article} % For LaTeX2e
\usepackage{nips15submit_e,times}
\usepackage{url}
\usepackage{amsmath,amsfonts,amsthm,amssymb}
\usepackage[noend]{algpseudocode}
\usepackage{algorithm}

\usepackage{footnote}
\usepackage{color,xcolor}
\usepackage{graphicx,subfigure}
\usepackage{multirow}
\usepackage[font={footnotesize}]{caption}

\usepackage[title]{appendix}
\usepackage{multibib}
\newcites{app}{References}

\title{Fast and Guaranteed Tensor Decomposition via Sketching}

\author{
Yining Wang, Hsiao-Yu Tung, Alex Smola \\
Machine Learning Department\\
Carnegie Mellon University, 
Pittsburgh, PA 15213 \\
\texttt{\{yiningwa,htung\}@cs.cmu.edu} \\
\texttt{alex@smola.org}
\And
Anima Anandkumar\\
Department of EECS\\
University of California Irvine\\
Irvine, CA 92697\\
\texttt{a.anandkumar@uci.edu}
}

\nipsfinalcopy % Uncomment for camera-ready version

\newcommand{\vct}{\boldsymbol }
\newcommand{\mat}{\mathbf}

\newcommand{\nml}{\mathcal{N}}

\newcommand{\argmin}{\mathrm{argmin}}
\newcommand{\argmax}{\mathrm{argmax}}
\newcommand{\poly}{\mathrm{poly}}

\newcommand{\median}{\mathrm{med}}

\newtheorem{thm}{Theorem}
\newtheorem{lem}{Lemma}

\newtheorem{prop}{Proposition}

\newtheorem{defn}{Definition}

\newcommand{\inner}[2]{\left\langle #1,#2 \right\rangle}
\newcommand{\rbr}[1]{\left(#1\right)}
\newcommand{\sbr}[1]{\left[#1\right]}
\newcommand{\cbr}[1]{\left\{#1\right\}}
\newcommand{\nbr}[1]{\left\|#1\right\|}
\newcommand{\abr}[1]{\left|#1\right|}

\begin{document}

\allowdisplaybreaks[4]

\maketitle

\begin{abstract}
Tensor CANDECOMP/PARAFAC (CP) decomposition has wide applications in statistical learning of latent variable models and in data mining.
In this paper, we propose fast and randomized tensor CP decomposition algorithms based on sketching. We build on the idea of count sketches, but introduce many novel ideas
%{\color{red} such as directly reading results from FFT output}
which are unique to tensors.  We develop novel methods for randomized computation of   tensor contractions via FFTs, without explicitly forming the tensors. Such tensor contractions   are encountered in decomposition  methods such as tensor power iterations and alternating least squares. We also  design  novel colliding hashes for symmetric tensors to further save time in computing the sketches.  We then combine these sketching ideas with   existing whitening and tensor power iterative   techniques to obtain the fastest algorithm  on both sparse and dense tensors. The quality of approximation under our method does not depend on properties such as sparsity, uniformity of elements, etc. We apply the method for topic modeling and obtain competitive results.

\end{abstract}

\vspace{-0.2cm}
\section{Introduction}
\vspace{-0.1cm}

% explain tensor CP decomposition

%, computed from samples,

In many data-rich domains such as computer vision, neuroscience and social networks consisting of multi-modal and multi-relational data, tensors have emerged as a powerful paradigm for handling the data deluge. An important operation with tensor data is its  decomposition, where the input tensor is decomposed into a succinct form. One of the popular decomposition methods is the  CANDECOMP/PARAFAC (CP) decomposition, also known as canonical polyadic decomposition \cite{cp-ref1,cp-ref2}, where the input tensor is decomposed into a succinct sum of rank-$1$ components.
The CP decomposition has found numerous applications in data mining \cite{cp-knowledge-base,cp-dblp,cp-webgraph},
computational neuroscience \cite{cp-topographic,cp-eeg},
and recently, in statistical learning for latent variable models \cite{tensor-power-method,spectral-slda,spectral-ibp,spectral-graphical-model}.
For latent variable modeling, these  methods yield consistent estimates under mild conditions such as non-degeneracy and require only  polynomial sample and computational complexity \cite{tensor-power-method,spectral-slda,spectral-ibp,spectral-graphical-model}.

% related work. Emphasize the difference

Given the importance of tensor methods for large-scale machine learning, there has been an increasing interest in scaling up tensor decomposition algorithms to handle gigantic real-world data tensors~\cite{mach,fast-als,distributed-als,gigatensor,sgd-spec,fast-tensor-sampling,scalable-spectral-lda}.
However, the previous works fall short in many ways, as described subsequently.
In this paper, we  design and analyze  efficient  randomized tensor methods using ideas from sketching~\cite{cs-poly-kernel}.  The idea is to maintain a low-dimensional sketch of an input tensor and then perform implicit tensor decomposition using existing methods such as tensor power updates, alternating least squares or online tensor updates. We  obtain the fastest decomposition methods for both sparse and dense tensors. Our framework can easily handle modern machine learning applications with billions of training instances, and at the same time, comes with attractive theoretical guarantees.
%
%We design and analyze randomized tensor methods using ideas from compressed sensing and sketching.
%In our framework we do not need to explicitly construct the tensor, which results in huge computational savings. For instance, in spectral learning latent variable models, empirical moment tensors need to be formed from samples, which can be as large as billions in modern machine learning applications.  Our proposed method can easily handle such challenging settings, while at the same time, comes with attractive theoretical guarantees.

%\footnote{Due to space constraints, formulation and experiments of fast ALS are presented in Appendix \ref{appsec:fast_als}.} based on the idea of tensor sketch \cite{cs-poly-kernel}.In particular, we maintain a low-dimensional sketch of an input tensor $\mat T$ and perform tensor power updates using the sketch.
%Note that under certain cases, the input tensor  does not need to be explicitly constructed, as we show in Sec. \ref{sec:tensor_sketch_build}
%Note that we do not always need to explicitly construct the full $k\times k\times k$ tensor in order to obtain its sketch.
%In Sec. \ref{sec:tensor_sketch_build} we show that much more efficient sketch construction algorithms exist when an input tensor has certain factored forms.
%\yining{One of the ICML reviewers complain that it is unclear from the introduction whether the full input tensor needs to be explicitly computed.}\aacomment{yes, this is good!}

Our main contributions are as follows:
%%\vspace{-0.05cm}
%\begin{itemize}
%%\vspace{-0.05cm}
%\item
\paragraph{Efficient tensor sketch construction: }We propose efficient construction of tensor sketches  when the input tensor is available in   factored forms such as in the case of empirical moment tensors, where  the factor components correspond to rank-$1$ tensors over individual data samples. We construct the tensor sketch via efficient FFT operations on the component vectors. %\aacomment{mention exact computation time here}
Sketching each rank-$1$ component takes $O(n+b\log b)$ operations where $n$ is the tensor dimension and $b$ is the sketch length.
This is much faster than the $O(n^p)$ complexity for brute force computations of a $p$th-order tensor.
Since empirical moment tensors are available in the factored form with $N$ components, where $N$ is the number of samples, it takes $O((n+b\log b)N)$ operations to compute the sketch.

\vspace{-0.2cm}
\paragraph{Implicit tensor contraction computations: }Almost all tensor manipulations can be expressed in terms of tensor {\em contractions}, which  involves multilinear combinations of different tensor {\em fibres}~\cite{tensor-review}.
For example, tensor decomposition methods such as tensor power iterations, alternating least squares  (ALS), whitening and online tensor methods all involve tensor contractions. We propose a highly efficient method to directly compute the tensor contractions without forming the input tensor explicitly.
In particular, given the sketch of a tensor, each tensor contraction can be computed in $O(n+b\log b)$ operations, 
regardless of order of the source and destination tensors.
This significantly accelerates the brute-force implementation that requires $O(n^p)$ complexity for $p$th-order tensor contraction.
In addition, in many applications, the input tensor is not directly available and needs to be computed from samples, such as the case of empirical moment tensors for spectral learning of latent variable models.
In such cases,  our method results in huge savings by combining implicit tensor contraction computation with efficient tensor sketch construction.
  %\aacomment{mention exact computation time}\aacomment{also in the table u have on computational complexity, can u have two more rows: exact decomp + whitening and whitening+power method with sketching.. it is good to see this explicitly}

\vspace{-0.2cm}
\paragraph{Novel colliding hashes for symmetric tensors: }
When the input tensor is symmetric, which is the case   for   empirical moment tensors that arise in  spectral learning applications,
we propose a novel colliding hash design by replacing the Boolean ring with the complex ring $\mathbb C$ to handle multiplicities.
As a result, it makes the sketch building process much faster and avoids repetitive FFT operations.
Though the computational complexity remains the same,
the proposed colliding hash design results in significant speed-up in practice by reducing the actual number of computations.
%\aacomment{mention exact reduction in computational complexity}

\vspace{-0.2cm}
\paragraph{Theoretical and empirical guarantees: }We show that  the quality of the tensor sketch does not depend on sparseness, uniform entry distribution, or any other properties of the input tensor. On the other hand, previous works assume specific settings such as   sparse tensors~\cite{fast-als,distributed-als,gigatensor}, or   tensors having entries with similar magnitude~\cite{mach}. Such assumptions are unrealistic, and in practice,   we may have both dense and spiky tensors, for example, unordered word trigrams in natural language processing. We prove that our proposed randomized method for tensor decomposition does not lead to any significant degradation of accuracy.

Experiments on synthetic and real-world datasets show highly competitive results.
We demonstrate a 10x to 100x speed-up over exact methods %like the Matlab tensor toolbox \cite{tensor-toolbox-1,tensor-toolbox-2}
for decomposing dense, high-dimensional tensors.
For topic modeling, we show a significant reduction in computational time over existing spectral LDA implementations
with small performance loss.
In addition, our proposed algorithm outperforms collapsed Gibbs sampling when running time is constrained.
%Apart from being efficient and effective, our proposed algorithm also serves as a good initializer for state-of-the-art methods based on collapsed Gibbs sampling.
We also show that if a Gibbs sampler is initialized with our output topics,
it converges within several iterations and outperforms a randomly initialized Gibbs sampler run for much more iterations.
%We also show that if a Gibbs sampler is initialized with the dictionary output by our algorithm,
%within 5 iterations it outputs better solutions than the one obtained by a randomly initialized Gibbs sampler running for hundreds of iterations.
Since our proposed method is efficient and avoids local optima, it can be used to accelerate the slow burn-in phase in Gibbs sampling.

\vspace{-0.2cm}
\paragraph{Related Works: }There have been many works on deploying efficient tensor decomposition methods~\cite{mach,fast-als,distributed-als,gigatensor,sgd-spec,fast-tensor-sampling,scalable-spectral-lda}. Most of these works except \cite{mach,fast-tensor-sampling} implement the  alternating least squares (ALS) algorithm \cite{cp-ref1,cp-ref2}. However,  this is  extremely expensive since the ALS method is run in the input space, which requires $O(n^3)$ operations to execute one least squares step on an $n$-dimensional (dense) tensor. Thus, they are only suited for extremely sparse tensors.

An alternative  method is to first reduce the dimension of the input tensor through procedures such as {\em whitening} to $O(k)$ dimension, where $k$ is the tensor rank, and then carry out ALS in the dimension-reduced space on $k\times k\times k$ tensor~\cite{anima-distributed-lda}. This results  in significant reduction of computational complexity when the rank is small $(k\ll n)$.  Nonetheless, in practice, such complexity is still prohibitively high as $k$ could be several thousands  in many settings. To make matters even worse, when the tensor corresponds to empirical moments computed from samples, such as in  spectral learning of latent variable models,
it is actually much slower to construct the reduced dimension $k\times k\times k$ tensor from training data than to  decompose it, since the number of training samples is typically very large.  Another alternative is to carry out online tensor decomposition, as opposed to batch operations in the above works. Such methods are extremely fast~\cite{sgd-spec}, but can suffer from high variance. The sketching ideas developed in this paper will improve our ability to handle larger sizes of mini-batches and therefore result in reduced variance in online tensor methods.

Another alternative method is to consider a randomized sampling of the input tensor in each iteration of tensor decomposition~\cite{mach,fast-tensor-sampling}. %\aacomment{however, these can be expensive due to I/O calls. is this true? we need to say something}
However, such methods can be expensive due to I/O calls and are sensitive to the sampling distribution.
In particular, \cite{mach} employs uniform sampling, which is incapable of handling tensors with spiky elements.
Though non-uniform sampling is adopted in \cite{fast-tensor-sampling}, it requires an additional pass over the training data to compute the sampling distribution.
In contrast, our sketch based method takes only one pass of the data.

%\aacomment{how about these works? \cite{mach,fast-tensor-sampling} we need to say how we are better than them? I thought there was a plan to try out tensor sampling in experiments. what happened with that?}

 %
%Note that the whitening procedure followed by robust power method is guaranteed to yield the correct decomposition when it exists.
%Under certain conditions, the computation could be reduced to $O(k^3)$ with $k$ the intrinsic rank of the input tensor by applying the whitening trick \cite{tensor-power-method}.

\vspace{-0.15cm}
\section{Preliminaries}
\vspace{-0.15cm}

%In this section we first review tensor related concepts and provide an overview of two popular algorithms for tensor CP decomposition.
%We also review the count sketch and tensor sketch technique, which is introduced in \cite{count-sketch,compressed-matrix-multiplication} and inspires this work.
%Notations for vector/matrix products are defined in Table \ref{tab_notation}.

\paragraph{Tensor, tensor product and tensor decomposition}%
%
%For two complex vectors $\vct a,\vct b\in\mathbb C^n$, we use $\vct a*\vct b\in\mathbb C^n$ to denote the convolution between them.
%Namely, $\vct a*\vct b\in\mathbb C^d$ and $[\vct a*\vct b]_n = \sum_{i=0}^{d-1}{\vct a_i\vct b_{(n-i)\mod d}}.$
%We use $\vct a\circ\vct b$ to denote the entrywise product (Hadamard product) of $\vct a$ and $\vct b$
%and $\langle\vct a,\vct b\rangle\in\mathbb R$ for their complex inner product, that is,
%$\langle\vct a,\vct b\rangle = \sum_{i=1}^n{\vct a_i\overline{\vct b_i}} \in\mathbb C.$
%We use $\mathcal F(\cdot)$ and $\mathcal F^{-1}(\cdot)$ to denote the FFT and inverse FFT operators.
%
%A $p$th order tensor $\mat A\in\bigotimes^p\mathbb R^n$ of dimension $n$ can be indexed by an $n$-tuple $(i_1,\cdots,i_p)$
%where $i_1,\cdots,i_p\in[n]$.
%In this work we mainly consider 3rd order tensors ($p=3$), though the technique can be easily extended to tensors of higher order.
A 3rd order tensor
\footnote{Though we mainly focus on 3rd order tensors in this work, extension to higher order tensors is easy.}
$\mat T$ of dimension $n$ has $n^3$ entries.
Each entry can be represented as $\mat T_{ijk}$ for $i,j,k\in\{1,\cdots,n\}$.
For an $n\times n\times n$ tensor $\mat T$ and a vector $\vct u\in\mathbb R^n$, we define two forms of tensor products (contractions) as follows:
 \abovedisplayskip=3pt \belowdisplayskip=3pt
\begin{equation*}
\mat T(\vct u,\vct u,\vct u) = \sum_{i,j,k=1}^n{\mat T_{i,j,k}\vct u_i\vct u_j\vct u_k};\;\;
\mat T(\mat I,\vct u,\vct u) = \left[\sum_{j,k=1}^n{\mat T_{1,j,k}\vct u_j\vct u_k}, \cdots, \sum_{j,k=1}^n{\mat T_{n,j,k}\vct u_j\vct u_k}\right].
\end{equation*}
Note that $\mat T(\vct u,\vct u,\vct u)\in\mathbb R$ and $\mat T(\mat I,\vct u,\vct u)\in\mathbb R^n$.
For two complex tensors $\mat A,\mat B$ of the same order and dimension,
its inner product is defined as $\langle\mat A,\mat B\rangle := \sum_{\vct l}{\mat A_{\vct l}\overline{\mat B}_{\vct l}}$,
where $\vct l$ ranges over all tuples that index the tensors.
The Frobenius norm of a tensor is simply $\|\mat A\|_F = \sqrt{\langle\mat A,\mat A\rangle}$.

%For a 3rd order tensor $\mat T\in\mathbb R^{n\times n\times n}$ its rank-$k$ CP decomposition involves
The \emph{rank-$k$ CP decomposition} of a 3rd-order $n$-dimensional tensor $\mat T\in\mathbb R^{n\times n\times n}$ involves scalars $\{\lambda_i\}_{i=1}^k$
and $n$-dimensional vectors $\{\vct a_i,\vct b_i,\vct c_i\}_{i=1}^k$
such that the residual
$\|\mat T-\sum_{i=1}^k{\lambda_i\vct a_i\otimes\vct b_i\otimes\vct c_i}\|_F^2$
is minimized.
Here $\mat R=\vct a\otimes\vct b\otimes\vct c$ is a 3rd order tensor defined as
$\mat R_{ijk} = \vct a_i\vct b_j\vct c_k$.
Additional notations are defined in Table \ref{tab_notation} and Appendix \ref{appsec:notation}.

\begin{table}[t]
\centering
\caption{Summary of notations. See also Appendix \ref{appsec:notation}.}
%\vskip 0.05in
\scalebox{0.85}{
\begin{tabular}{llllll}
\hline
{Variables}& Operator& Meaning& {Variables}& Operator& Meaning\\
\hline
$\vct a,\vct b\in\mathbb C^n$& $\vct a\circ\vct b\in\mathbb C^n$& Element-wise product& $\vct a\in\mathbb C^n$& $\vct a^{\otimes 3}\in\mathbb C^{n\times n\times n}$& $\vct a\otimes\vct a\otimes\vct a$\\
$\vct a,\vct b\in\mathbb C^n$& $\vct a*\vct b\in\mathbb C^n$& Convolution& $\mat A,\mat B\in\mathbb C^{n\times m}$& $\mat A\odot\mat B\in\mathbb C^{n^2\times m}$& Khatri-Rao product\\
$\vct a,\vct b\in\mathbb C^n$& $\vct a\otimes\vct b\in\mathbb C^{n\times n}$& Tensor product& $\mat T\in\mathbb C^{n\times n\times n}$& $\mat T_{(1)}\in\mathbb C^{n\times n^2}$& Mode expansion\\
\hline
\end{tabular}
}
\label{tab_notation}
\vspace*{-0.4cm}
\end{table}

%\vspace{-0.3cm}
\paragraph{Robust tensor power method}
The method was proposed in \cite{tensor-power-method} and was shown to provably succeed
if the input tensor is a noisy perturbation of the sum of $k$ rank-1 tensors whose base vectors are orthogonal.
%Algorithm \ref{alg_tensor_power} displays the pseudocode of the method.
Fix an input tensor $\mat T\in\mathbb R^{n\times n\times n}$,
The basic idea is to randomly generate $L$ initial vectors and perform $T$ power update steps:
$\hat{\vct u} = \mat T(\mat I,\vct u,\vct u)/\|\mat T(\mat I,\vct u,\vct u)\|_2.$
The vector that results in the largest eigenvalue $\mat T(\vct u,\vct u,\vct u)$ is then kept and
subsequent eigenvectors can be obtained via deflation.
If implemented naively, the algorithm takes $O(kn^3LT)$ time to run
\footnote{$L$ is usually set to be a linear function of $k$ and $T$ is logarithmic in $n$; see Theorem 5.1 in \cite{tensor-power-method}.}, requiring $O(n^3)$ storage.
In addition, in certain cases when a second-order moment matrix is available,
the tensor power method can be carried out on a $k\times k\times k$ whitened tensor \cite{tensor-power-method},
thus improving the time complexity by avoiding dependence on the ambient dimension $n$.
Apart from the tensor power method, other algorithms such as Alternating Least Squares (ALS, \cite{cp-ref1,cp-ref2})
and Stochastic Gradient Descent (SGD, \cite{sgd-spec}) have also been applied to tensor CP decomposition.

\paragraph{Tensor sketch} \emph{Tensor sketch} was proposed in \cite{cs-poly-kernel}
as a generalization of count sketch \cite{count-sketch}.
%It has been successfully applied to fast matrix multiplication and kernel computations.
For a tensor $\mat T$ of dimension $n_1\times\cdots\times n_p$,
random hash functions $h_1,\cdots,h_p:[n]\to[b]$ with $\Pr_{h_j}[h_j(i)=t]=1/b$ for every $i\in[n], j\in[p], t\in[b]$
and binary Rademacher variables $\xi_1,\cdots,\xi_p:[n]\to\{\pm 1\}$,
%and random Bernoulli variables $\xi_1,\cdots,\xi_p:[n]\to\{+1,-1\}$ with $\Pr_{\xi_j}[\xi_j(i)=1]=\Pr_{\xi_j}[\xi_j(i)=-1]=1/2$,
the sketch $s_{\mat T}:[b]\to\mathbb R$ of tensor $\mat T$ is defined as
\begin{equation}
s_{\mat T}(t) = \sum_{H(i_1,\cdots,i_p)=t}{\xi_1(i_1)\cdots\xi_p(i_p)\mat T_{i_1,\cdots,i_p}},
\label{eq_asym_tensor_sketch}
\end{equation}
where $H(i_1,\cdots,i_p) = (h_1(i_1)+\cdots+h_p(i_p))\mod b$.
The corresponding recovery rule is $\widehat{\mat T}_{i_1,\cdots,i_p} = \xi_1(i_1)\cdots\xi_p(i_p)s_{\mat T}(H(i_1,\cdots,i_p))$.
For accurate recovery, $H$ needs to be 2-wise independent, which is achieved by independently selecting $h_1,\cdots,h_p$ from a 2-wise independent hash family \cite{universal-hash-2}.
Finally, the estimation can be made more robust by the standard approach of taking $B$ independent sketches of the same tensor
and then report the median of the $B$ estimates \cite{count-sketch}.

\vspace{-0.25cm}
\section{Fast tensor decomposition via sketching}\label{sec:tensor_decomposition}
\vspace{-0.25cm}

%\aacomment{this section is disconnected from previous result. you should say that in previous result, we specified accuracy of the sketch, but construction of the sketch can be expensive, in this section we improve this for special class of tensors}

%In this section we present the main algorithm (fast robust tensor power method) for tensor decomposition.\aacomment{we are not presenting tensor decomposition algorithm.. this paragraph is confusing. You can say we provide efficient sketch construction and then show how to run power method directly on the sketch with reduced computational complexity.. discussion of ALS not needed here, it will be after the power method.}
%Fast ALS method is presented in Appendix \ref{appsec:fast_als} due to space limits.
%The proposed algorithm achieve significant improvement in terms of both running time and memory storage.
%In this section we first introduce tensor sketching \cite{cs-poly-kernel} and show how sketches can be computed efficiently for factored or empirical moment tensors.
In this section we first introduce an efficient procedure for computing sketches of factored or empirical moment tensors,
which appear in a wide variety of applications such as parameter estimation of latent variable models.
We then show how to run tensor power method directly on the sketch with reduced computational complexity.
In addition, when an input tensor is symmetric (i.e., $\mat T_{ijk}$ the same for all permutations of $i,j,k$)
we propose a novel ``colliding hash" design, which speeds up the sketch building process.
Due to space limits we only consider the robust tensor power method in the main text.
Methods and experiments for sketching based ALS are presented in Appendix \ref{appsec:fast_als}.
%over the original tensor sketch design.

To avoid confusions, we emphasize that $n$ is used to denote the dimension of the tensor \emph{to be decomposed},
which is not necessarily the same as the dimension of the original data tensor.
Indeed, once whitening is applied $n$ could be as small as the intrinsic dimension $k$ of the original data tensor.

%\vspace{-0.25cm}
%\subsection{Tensor sketch}
%\vspace{-0.2cm}
%Sketching is a power method for dimensionality reduction and fast statistical computation.
%We start by describing \emph{count sketch} proposed in \cite{count-sketch}.
%Suppose $b$ is the target hash length.
%For a 2-wise independent hash function $h:[n]\to[b]$ and independent Rademacher variables $\xi:[n]\to\{+1,-1\}$,
%the count sketch of a vector $\vct x\in\mathbb R^n$ is a mapping $s_{\vct x}:[b]\to\mathbb R$ defined as
%\footnote{To simplify notations, we use $\vct s_{\vct x}=[s_{\vct x}(1),\cdots,s_{\vct x}(b)]$ to denote the vector version of $s_{\vct x}$.}
%\begin{equation}
%s_{\vct x}(t) = \sum_{h(i)=t}{\xi(i)\vct x_i},\quad\forall t=1,\cdots,b.
%\end{equation}
%It can be shown that $\hat{\vct x}_i = \xi(i)s_{\vct x}(h(i))$ is an unbiased estimator of $\vct x_i$.
%That is, $\mathbb E_{h,\xi}[\hat{\vct x}_i] = \vct x_i$.

%It is well known that if $h_1,\cdots,h_p$ are 3-wise independent hashes then so is $H$ \cite{universal-hash-1,universal-hash-2}.

\vspace{-0.15cm}
\subsection{Efficient sketching of empirical moment tensors}\label{sec:tensor_sketch_build}
\vspace{-0.15cm}

Sketching a  3rd-order dense $n$-dimensional tensor via Eq.~(\ref{eq_asym_tensor_sketch}) takes $O(n^3)$ operations,
which in general cannot be improved because the input size is $\Omega(n^3)$.
However, in practice data tensors are usually structured.
One notable example is \emph{empirical moment tensors}, which arises naturally in parameter estimation problems of latent variable models.
More specifically, an empirical moment tensor can be expressed as $\mat T=\hat{\mathbb E}[\vct x^{\otimes 3}] = \frac{1}{N}\sum_{i=1}^N{\vct x_i^{\otimes 3}}$,
where $N$ is the total number of training data points and $\vct x_i$ is the $i$th data point.
In this section we show that computing sketches of such tensors can be made significantly more efficient than the brute-force implementations via Eq. (\ref{eq_asym_tensor_sketch}).
The main idea is to sketch low-rank components of $\mat T$ efficiently via FFT, a trick inspired by previous efforts on sketching based matrix multiplication and kernel learning
\cite{compressed-matrix-multiplication,cs-poly-kernel}.

We consider the more generalized case when an input tensor $\mat T$ can be written as a weighted sum of known rank-1 components:
$\mat T=\sum_{i=1}^N{a_i\vct u_i\otimes\vct v_i\otimes\vct w_i}$,
where $a_i$ are scalars and $\vct u_i,\vct v_i,\vct w_i$ are known $n$-dimensional vectors.
The key observation is that the sketch of each rank-1 component $\mat T_i=\vct u_i\otimes\vct v_i\otimes\vct w_i$ can be efficiently computed by FFT.
In particular, $\vct s_{\mat T_i}$ can be computed as
\begin{equation}
\vct s_{\mat T_i} = \vct s_{1,\vct u_i} * \vct s_{2,\vct v_i} * \vct s_{3,\vct w_i}
= \mathcal F^{-1}(\mathcal F(\vct s_{1,\vct u_i})\circ\mathcal F(\vct s_{2,\vct v_i})\circ\mathcal F(\vct s_{3,\vct w_i})),
\label{eq_asym_rank_one_update}
\end{equation}
where $*$ denotes convolution and $\circ$ stands for element-wise vector product.
$\vct s_{1,\vct u}(t) = \sum_{h_1(i)=t}{\xi_1(i)\vct u_i}$ is the count sketch of $\vct u$ and $\vct s_{2,\vct v}, \vct s_{3,\vct w}$
are defined similarly.
$\mathcal F$ and $\mathcal F^{-1}$ denote the Fast Fourier Transform (FFT) and its inverse operator.
By applying FFT, we reduce the convolution computation into element-wise product evaluation in the Fourier space.
Therefore, $\vct s_{\mat T}$ can be computed using $O(n+b\log b)$ operations,
where the $O(b\log b)$ term arises from FFT evaluations.
Finally, because the sketching operator is linear (i.e., $\vct s(\sum_i{a_i\mat T_i}) = \sum_i{a_i\vct s(\mat T_i)}$),
$\vct s_{\mat T}$ can be computed in $O(N(n+b\log b))$,
which is much cheaper than brute-force that takes $O(Nn^3)$ time.

\vspace{-0.2cm}
\subsection{Fast robust tensor power method}\label{subsec:fastrbp}
\vspace{-0.2cm}

\setlength{\textfloatsep}{5pt}
\begin{algorithm}[t]
\caption{Fast robust tensor power method}
\begin{algorithmic}[1]
\State \textbf{Input}: noisy symmetric tensor $\bar{\mat T}=\mat T+\mat E\in\mathbb R^{n\times n\times n}$; target rank $k$;
number of initializations $L$, number of iterations $T$, hash length $b$, number of independent sketches $B$.
%\aacomment{need to define all notations.. $L$ is number of initializations.. as a rule of thumb, even if the reader has not read a single word in the paper, he/she should be able to understand the pseudocode}
\State \textbf{Initialization}: $h_j^{(m)},\xi_j^{(m)}$ for $j\in\{1,2,3\}$ and $m\in[B]$;
compute sketches ${\vct s}_{\bar{\mat T}}^{(m)}\in\mathbb C^b$.
%compute $\tilde{\vct s}^{(m)}_{\bar{\mat T}}\in\mathbb C^b$ for $m=1,\cdots,B$.
\For{$\tau=1$ to $L$}
	\State Draw $\vct u_0^{(\tau)}$ uniformly at random from unit sphere.
	% {\color{red} Build $\vct s_{\mat T}(\vct u_0^{(\tau)})$ with $p=1$ via Eq. (\ref{eq_asym_tensor_sketch}).}
	\For{$t=1$ to $T$}
		\State For each $m\in[B],j\in\{2,3\}$ compute the sketch of $\vct u_{t-1}^{(\tau)}$ using $h_j^{(m)}$,$\xi_j^{(m)}$ via Eq. (\ref{eq_asym_tensor_sketch}).
		\State Compute $\vct v^{(m)} \approx \bar{\mat T}(\mat I,\vct u_{t-1}^{(\tau)}, \vct u_{t-1}^{(\tau)})$ as follows:
		first evaluate $\bar{\vct s}^{(m)}=\mathcal F^{-1}(\mathcal F(\vct s_{\bar{\mat T}}^{(m)})\circ\overline{\mathcal F(\vct s_{2,\vct u}^{(m)})}\circ\overline{\mathcal F(\vct s_{3,\vct u}^{(m)})})$.
		%where $\mathcal F,\mathcal F^{-1}$ stand for Fourier and inverse Fourier transform.
		Set $[\vct v^{(m)}]_i$ as $[\vct v^{(m)}]_i \gets\xi_1(i)[\bar{\vct s}^{(m)}]_{h_1(i)}$ for every $i\in[n]$.
\State Set $\bar{\vct v}_i \gets \median(\Re(\vct v^{(1)}_i), \cdots, \Re(\vct v^{(B)}_i))$\footnotemark.
		Update: $\vct u_t^{(\tau)} = \bar{\vct v} / \|\bar{\vct v}\|$.		
	\EndFor
\EndFor
\State \textbf{Selection} Compute $\lambda_{\tau}^{(m)} \approx \bar{\mat T}(\vct u_T^{(\tau)},\vct u_T^{(\tau)}, \vct u_T^{(\tau)})$
using ${\vct s}^{(m)}_{\bar{\mat T}}$ for $\tau\in[L]$ and $m\in[B]$.
Evaluate $\lambda_{\tau}=\median(\lambda_{\tau}^{(1)},\cdots,\lambda_{\tau}^{(B)})$ and $\tau^* = \argmax_{\tau}{\lambda_{\tau}}$.
Set $\hat\lambda=\lambda_{\tau^*}$ and $\hat{\vct u}=\vct u_T^{(\tau^*)}$.
%$\tau^* \gets \argmax_{\tau}{\median(\Re(\lambda_{\tau}^{(1)}),\cdots, \Re(\lambda_{\tau}^{(B)}))}$.
%\State Do another $T$ power updates starting from $\vct u_T^{(\tau^*)}$ to obtain $\hat{\vct u}$; compute $\hat\lambda \approx \bar{\mat T}(\hat{\vct u},\hat{\vct u},\hat{\vct u})$ approximately.
\State\textbf{Deflation} %Estimate $\hat\lambda$ by approximately computing $\bar{\mat T}(\hat{\vct u},\hat{\vct u},\hat{\vct u})$.
For each $m\in[B]$ compute sketch $\tilde{\vct s}_{\Delta\mat T}^{(m)}$ for the rank-1 tensor $\Delta\mat T = \hat{\lambda}\hat{\vct u}^{\otimes 3}$.
\State \textbf{Output}: the eigenvalue/eigenvector pair $(\hat\lambda,\hat{\vct u})$ and sketches of the deflated tensor $\bar{\mat T}-\Delta\mat T$.
\end{algorithmic}
\label{alg_fast_rbp}
%%\vspace*{-0.7cm}
\end{algorithm}
\footnotetext{$\Re(\cdot)$ denotes the real part of a complex number. $\median(\cdot)$ denotes the median.}

%\aacomment{we need to show a table comparing computational complexity of plain power method and ours.}

We are now ready to present the fast robust tensor power method, the main algorithm of this paper.
%Pseudocodes are listed in Algorithm \ref{alg_fast_rbp}.
The computational bottleneck of the original robust tensor power method is the computation of two tensor products:
$\mat T(\mat I,\vct u,\vct u)$ and $\mat T(\vct u, \vct u, \vct u)$.
A naive implementation requires $O(n^3)$ operations.
In this section, we show how to speed up computation of these products.
We show that given the sketch of an input tensor $\mat T$,
one can approximately compute both $\mat T(\mat I,\vct u,\vct u)$ and $\mat T(\vct u,\vct u, \vct u)$
in $O(b\log b+n)$ steps, where $b$ is the hash length.
%For simplicity, we assume only one sketch is used in this and the following sections.
%Extension to multiple independent sketches is straight-forward
%and is detailed in Algorithm \ref{alg_fast_rbp} and \ref{alg_fast_als}.

Before going into details, we explain the key idea behind our fast tensor product computation.
For any two tensors $\mat A,\mat B$,
its inner product $\langle\mat A,\mat B\rangle$ can be approximated by
\footnote{All approximations will be theoretically justified in Section \ref{sec:analysis} and Appendix \ref{appsec:tensor_sketch_bound}.}
\begin{equation}
\langle\mat A,\mat B\rangle \approx \langle\vct s_{\mat A},\vct s_{\mat B}\rangle.
\label{eq_ab}
\end{equation}
%Note that the inner product on the right hand side of Eq. (\ref{eq_ab_collide}) is complex inner product.
Eq. (\ref{eq_ab}) immediately results in a fast approximation procedure of $\mat T(\vct u,\vct u,\vct u)$
because $\mat T(\vct u,\vct u,\vct u) = \langle\mat T,\mat X\rangle$ where $\mat X=\vct u\otimes\vct u\otimes\vct u$
is a rank one tensor, whose sketch can be built in $O(n+b\log b)$ time by Eq.~(\ref{eq_asym_rank_one_update}).
Consequently, the product can be approximately computed using $O(n+b\log b)$ operations if the tensor sketch of $\mat T$ is available.
For tensor product of the form $\mat T(\mat I,\vct u,\vct u)$.
The $i$th coordinate in the result can be expressed as $\langle\mat T,\mat Y_i\rangle$ where $\mat Y_i=\vct e_i\otimes\vct u\otimes\vct u$;
$\vct e_i=(0,\cdots,0,1,0,\cdots,0)$ is the $i$th indicator vector.
We can then apply Eq. (\ref{eq_ab}) to approximately compute $\langle\mat T,\mat Y_i\rangle$ efficiently.
However, this method is not completely satisfactory because it requires sketching $n$ rank-1 tensors ($\mat Y_1$ through $\mat Y_n$),
which results in $O(n)$ FFT evaluations by Eq. (\ref{eq_asym_rank_one_update}). 
%However, this method requires $O(n)$ FFT evaluations and is not completely satisfactory.
%\aacomment{why is this case? The FFT appears suddenly.. elaborate and give the exact operation}
Below we present a proposition that allows us to use only $O(1)$ FFTs to approximate $\mat T(\mat I,\vct u,\vct u)$.
{
\begin{prop}
$\langle\vct s_{\mat T},\vct s_{1,\vct e_i}*\vct s_{2,\vct u}*\vct s_{3,\vct u}\rangle=
\langle\mathcal F^{-1}(\mathcal F(\vct s_{\mat T})\circ\overline{\mathcal F(\vct s_{2,\vct u})}\circ\overline{\mathcal F(\vct s_{3,\vct u})}), \vct s_{1,\vct e_i}\rangle.
$
\label{prop_tiuu}
\end{prop}
Proposition \ref{prop_tiuu} is proved in Appendix \ref{appsec:proof_technical}.
The main idea is to ``shift" all terms not depending on $i$ to the left side of the inner product
and eliminate the inverse FFT operation on the right side so that $\vct s_{\vct e_i}$ contains only one nonzero entry.
As a result, we can compute $\mathcal F^{-1}(\mathcal F(\vct s_{\mat T})\circ\overline{\mathcal F(\vct s_{2,\vct u})}\circ\overline{\mathcal F(\vct s_{3,\vct u})})$ once
and read off each entry of $\mat T(\mat I,\vct u,\vct u)$ in constant time.
In addition,
the technique can be further extended to symmetric tensor sketches, with details deferred to Appendix \ref{appsec:detail_tensor_power} due to space limits.
%Details are deferred to Appendix \ref{appsec:detail_tensor_power} due to space limits.
%The pseudocode for fast robust tensor power method is presented in Algorithm \ref{alg_fast_rbp}.
When operating on an $n$-dimensional tensor,
The algorithm requires $O(kLT(n+Bb\log b))$ running time (excluding the time for building $\tilde{\vct s}_{\bar{\mat T}}$) and $O(Bb)$ memory,
which significantly improves the $O(kn^3LT)$ time and $O(n^3)$ space complexity over the brute force tensor power method.
Here $L,T$ are algorithm parameters for robust tensor power method.
Previous analysis shows that $T=O(\log k)$ and $L=\poly(k)$, where $\poly(\cdot)$ is some low order polynomial function.~\cite{tensor-power-method}
%\aacomment{so many notations undefined. need to highlight this and make it clear.. say from previous paper $T=O(\log k)$, from analysis below $B=\log (1/\delta)$.. $L=\poly(k)$, where $\poly$ is a low order polynomial.. leave no notation undefined . give the final bound for our method as a proposition}

%\aacomment{need a subsection on online tensor power method and discuss computational complexity of the overall method.}

Finally, Table \ref{tab_complexity} summarizes computational complexity of sketched and plain tensor power method.

\begin{table}[t]
\centering
\caption{Computational complexity of sketched and plain tensor power method.
$n$ is the tensor dimension; $k$ is the intrinsic tensor rank; $b$ is the sketch length. Per-sketch time complexity is shown.}
\scalebox{0.8}{
\begin{tabular}{ccccc}
\hline
& \textsc{Plain}& \textsc{Sketch}& \textsc{Plain+Whitening}& \textsc{Sketch+Whitening}\\
\hline
preprocessing: general tensors& -& $O(n^3)$& $O(kn^3)$& $O(n^3)$\\
preprocessing: factored tensors& \multirow{2}{*}{$O(Nn^3)$}& \multirow{2}{*}{$O(N(n+b\log b))$}& \multirow{2}{*}{$O(N(nk+k^3))$}& \multirow{2}{*}{$O(N(nk+b\log b))$}\\
with $N$ components\\
per tensor contraction time& $O(n^3)$& $O(n+b\log b)$& $O(k^3)$& $O(k+b\log b)$\\
%\hline
%\textsc{Plain}& -& -& $O(n^3)$\\
%\textsc{Sketch}& $O(n^3)$& $O(k(n+b\log b))$& $O(n+b\log b)$\\
%\textsc{Plain+Whitening}& \\
\hline
\end{tabular}
}
\label{tab_complexity}
%%\vspace{-0.1cm}
\end{table}

%\vspace{-0.1cm}
\subsection{Colliding hash and symmetric tensor sketch}\label{sec:colliding_hash}
%\vspace{-0.1cm}

{
For symmetric input tensors, it is possible to design a new style of tensor sketch that can be built more efficiently.
The idea is to design hash functions that deliberately collide symmetric entries, i.e., $(i,j,k)$, $(j,i,k)$, etc.
Consequently, we only need to consider entries $\mat T_{ijk}$ with $i\leq j\leq k$ when building tensor sketches.
An intuitive idea is to use the same hash function and Rademacher random variable for each order,
that is, $h_1(i)=h_2(i)=h_3(i)=:h(i)$ and $\xi_1(i)=\xi_2(i)=\xi_3(i)=:\xi(i)$.
In this way, all permutations of $(i,j,k)$ will collide with each other.
However, such a design has an issue with repeated entries because $\xi(i)$ can only take $\pm 1$ values.
Consider $(i,i,k)$ and $(j,j,k)$ as an example: $\xi(i)^2\xi(k)=\xi(j)^2\xi(k)$ with probability 1 even if $i\neq j$.
On the other hand, we need $\mathbb E[\xi(a)\xi(b)]=0$ for any pair of distinct 3-tuples $a$ and $b$.

To address the above-mentioned issue, we extend the Rademacher random variables to the complex domain and consider all roots of $z^m=1$,
that is, $\Omega=\{\omega_j\}_{j=0}^{m-1}$ where $\omega_j=e^{i\frac{2\pi j}{m}}$.
Suppose $\sigma(i)$ is a  Rademacher random variable with $\Pr[\sigma(i)=\omega_i]=1/m$.
By elementary algebra, $\mathbb E[\sigma(i)^p] = 0$ whenever $m$ is relative prime to $p$ or $m$ can be divided by $p$.
Therefore, by setting $m=4$ we avoid collisions of repeated entries in a 3rd order tensor.
}
More specifically,
The symmetric tensor sketch of a symmetric tensor $\mat T\in\mathbb R^{n\times n\times n}$ can be defined as
\begin{equation}
\tilde s_{\mat T}(t) := \sum_{\tilde H(i,j,k)=t}{\mat T_{i,j,k}\sigma(i)\sigma(j)\sigma(k)},
\label{eq_sym_tensor_sketch}
\end{equation}
where $\tilde H(i,j,k) = (h(i)+h(j)+h(k))\mod b$.
To recover an entry, we use
\begin{equation}
\widehat{\mat T}_{i,j,k} = 1/\kappa\cdot \overline{\sigma(i)}\cdot\overline{\sigma(j)}\cdot\overline{\sigma(k)}\cdot\tilde s_{\mat T}(H(i,j,k)),
\end{equation}
where $\kappa = 1$ if $i=j=k$; $\kappa = 3$ if $i=j$ or $j=k$ or $i=k$; $\kappa = 6$ otherwise.
For higher order tensors, the coefficients can be computed via the Young tableaux which characterizes symmetries under the permutation group.
Compared to asymmetric tensor sketches, the hash function $h$ needs to satisfy stronger independence conditions because we are using the same hash function for each order.
In our case, $h$ needs to be 6-wise independent to make $\tilde H$ 2-wise independent.
The fact is due to the following proposition, which is proved in Appendix \ref{appsec:proof_technical}.
\begin{prop}
Fix $p$ and $q$.
For $ h:[n]\to[b]$ define symmetric mapping $\tilde H:[n]^p\to[b]$ as $\tilde H(i_1,\cdots,i_p)= h(i_1)+\cdots+ h(i_p)$.
If $ h$ is $(pq)$-wise independent then $H$ is $q$-wise independent.
\label{prop_hash_independent}
\end{prop}
The symmetric tensor sketch described above can significantly speed up sketch building processes.
For a general tensor with $M$ nonzero entries, to build $\tilde{\vct s}_{\mat T}$ one only needs to consider roughly $M/6$ entries (those $\mat T_{ijk}\neq 0$ with $i\leq j\leq k$).
For a rank-1 tensor $\vct u^{\otimes 3}$, only one FFT is needed to build $\mathcal F(\tilde{\vct s})$;
in contrast, to compute Eq. (\ref{eq_asym_rank_one_update}) one needs at least 3 FFT evaluations.

Finally, in Appendix \ref{appsec:detail_tensor_power} we give details on how to seamlessly combine symmetric hashing
and techniques in previous sections to efficiently construct and decompose a tensor.

%It can be shown that if the period of the complex Rademacher random variables $m$ satisfies $m>3$ then
%for any 3-tuple $(i,j,k)$, the recovered value satisfies $\mathbb E_{h,\sigma}[\widehat{\mat T}_{i,j,k}] = \mat T_{i,j,k}$
%whenever $\mat T$ is symmetric.
}

\vspace{-0.3cm}
\section{Error analysis}\label{sec:analysis}
\vspace{-0.3cm}

In this section we provide theoretical analysis on approximation error of both tensor sketch and the fast sketched robust tensor power method.
We mainly focus on symmetric tensor sketches, while extension to asymmetric settings is trivial.
Due to space limits, all proofs are placed in the appendix.

\vspace{-0.2cm}
\subsection{Tensor sketch concentration bounds}\label{subsec:analysis_tensor_sketch}
\vspace{-0.2cm}

Theorem \ref{thm_tensor_error} bounds the approximation error of symmetric tensor sketches when computing $\mat T(\vct u,\vct u,\vct u)$ and $\mat T(\mat I,\vct u,\vct u)$.
Its proof is deferred to Appendix \ref{appsec:tensor_sketch_bound}.

\begin{thm}
Fix a symmetric real tensor $\mat T\in\mathbb R^{n\times n\times n}$ and a real vector $\vct u\in\mathbb R^n$ with $\|\vct u\|_2 = 1$.
Suppose $\varepsilon_{1,T}(\vct u)\in\mathbb R$ and $\vct\varepsilon_{2,T}(\vct u)\in\mathbb R^n$ are estimation errors of $\mat T(\vct u,\vct u,\vct u)$
and $\mat T(\mat I,\vct u,\vct u)$ using $B$ independent symmetric tensor sketches;
that is, $\varepsilon_{1,T}(\vct u) = \widehat{\mat T}(\vct u,\vct u,\vct u)-\mat T(\vct u,\vct u, \vct u)$
and $\vct\varepsilon_{2,T}(\vct u) = \widehat{\mat T}(\mat I,\vct u,\vct u)-\mat T(\mat I,\vct u, \vct u)$.
If $B=\Omega(\log(1/\delta))$ then with probability $\geq 1-\delta$ the following error bounds hold:
\begin{equation}
\big|\varepsilon_{1,T}(\vct u)\big| = O(\|\mat T\|_F/\sqrt{b});\quad
\big|\left[\vct\varepsilon_{2,T}(\vct u)\right]_i\big| = O(\|\mat T\|_F/\sqrt{b}),\;\;\forall i\in\{1,\cdots,n\}.
\end{equation}
In addition, for any fixed $\vct w\in\mathbb R^n$, $\|\vct w\|_2 = 1$ with probability $\geq 1-\delta$ we have
\begin{equation}
\left\langle\vct w, \vct\varepsilon_{2,T}(\vct u)\right\rangle^2 = O(\|\mat T\|_F^2/b).
\end{equation}
\label{thm_tensor_error}
\end{thm}
%\vspace*{-0.4cm}

\vspace{-0.5cm}
\subsection{Analysis of the fast tensor power method}\label{subsec:analysis_tensor_power_method}
\vspace{-0.2cm}

We present a theorem analyzing robust tensor power method with tensor sketch approximations.
A more detailed theorem statement along with its proof can be found in Appendix \ref{appsec:tensor_power_method}.

%\vspace{-0.1cm}
\begin{thm}
Suppose $\bar{\mat T}=\mat T+\mat E\in\mathbb R^{n\times n\times n}$ where
$\mat T=\sum_{i=1}^k{\lambda_i\vct v_i^{\otimes 3}}$ with an orthonormal basis $\{\vct v_i\}_{i=1}^k$, $\lambda_1>\cdots>\lambda_k>0$
and $\|\mat E\| = \epsilon$.
Let $\{(\hat\lambda_i,\hat{\vct v}_i)\}_{i=1}^k$ be the eigenvalue/eigenvector pairs obtained by Algorithm \ref{alg_fast_rbp}.
Suppose $\epsilon = O(1/(\lambda_1n))$, $T=\Omega(\log(n/\delta)+\log(1/\epsilon)\max_i{\lambda_i/(\lambda_i-\lambda_{i-1}}))$
and $L$ grows linearly with $k$.
Assume the randomness of the tensor sketch is independent among tensor product evaluations.
If $B=\Omega(\log(n/\delta))$ and $b$ satisfies
\begin{equation}
b = \Omega\left(\max\left\{\frac{\epsilon^{-2}\|\mat T\|_F^2}{\Delta(\vct\lambda)^2}, \frac{\delta^{-4}n^2\|\mat T\|_F^2}{r(\vct\lambda)^2\lambda_1^2}\right\}\right)
\end{equation}
where $\Delta(\vct\lambda) = \min_i(\lambda_i-\lambda_{i-1})$ and $r(\vct\lambda) = \max_{i,j>i}(\lambda_i/\lambda_j)$,
then with probability $\geq 1-\delta$ there exists a permutation $\pi$ over $[k]$ such that
\begin{equation}
\|\vct v_{\pi(i)}-\hat{\vct v}_i\|_2\leq\epsilon,\quad |\lambda_{\pi(i)}-\hat\lambda_i|\leq \lambda_i\epsilon/2,\;\;\forall i\in\{1,\cdots,k\}
\end{equation}
and $\|\mat T-\sum_{i=1}^k{\hat\lambda_i\hat{\vct v}_i^{\otimes 3}}\| \leq c\epsilon$ for some constant $c$.
\label{thm_tensor_power}
\end{thm}

Theorem \ref{thm_tensor_error} shows that the sketch length $b$ can be set as $o(n^3)$ to provably approximately decompose a 3rd-order tensor with dimension $n$.
Theorem \ref{thm_tensor_error} together with time complexity comparison in Table \ref{tab_complexity}
shows that the sketching based fast tensor decomposition algorithm has better computational complexity over brute-force implementation.
One potential drawback of our analysis is the assumption that sketches are independently built for each tensor product (contraction) evaluation.
This is an artifact of our analysis and we conjecture that it can be removed by incorporating recent development of differentially private adaptive query framework \cite{privacy-adaptive}. 

\vspace{-0.2cm}
\section{Experiments}\label{sec:experiment}
\vspace{-0.2cm}
%%\vspace{-0.3cm}

\begin{table*}[t]
%\begin{threeparttable}[t]
\centering
\caption{Squared residual norm on top 10 recovered eigenvectors of 1000d tensors %$\|\mat T-\sum_{i=1}^k{\lambda_i\vct v_i^{\otimes 3}}\|_F^2$
and running time (excluding I/O and sketch building time) for plain (exact) and sketched robust tensor power methods.
Two vectors are considered mismatch (wrong) if $\|\vct v-\hat{\vct v}\|_2^2 > 0.1$.
A extended version is shown as Table \ref{tab_synthtensor_acc} in Appendix \ref{appsec:supp_experiment}.
%``RB" stands for robust tensor power method.
}
%\vskip 0.05in
\scalebox{0.85}{
\begin{tabular}{ccccccccccccccccccc}
\hline
%\abovespace
 & & \multicolumn{5}{c}{Residual norm}& &\multicolumn{5}{c}{No. of wrong vectors}& &\multicolumn{5}{c}{Running time (min.)}\\
\cline{3-7}\cline{9-13}\cline{15-19}
&  $\log_2(b)$:&  12& 13& 14& 15&16&  & 12& 13& 14& 15&16&  &  12& 13& 14& 15& 16\\
\hline
\multirow{4}{*}{\rotatebox{90}{$\sigma=.01$}}
&$B=20$&.40&.19 &.10 &{\bf .09} &.08& & 8& 6& 3&{\bf 0} &0& & .85& 1.6& 3.5&{\bf 7.4}& 16.6 \\
&$B=30$&.26& .10&.09 &.08 &.07& & 7& 5& 2&0 &0& & 1.3& 2.4& 5.3&11.3& 24.6 \\
&$B=40$&.17& .10&{\bf .08} &.08 &.07 && 7& 4&{\bf 0}&0 &0& & 1.8& 3.3& {\bf 7.3}&15.2& 33.0 \\
& Exact& .07& & & & & & 0& & & & & & \multicolumn{5}{l}{293.5}\\
\hline
%\multirow{4}{*}{\rotatebox{90}{$\sigma=.1$}}
%&$B=20$&.52&3.1 &.21 &{\bf.18} &.17& & 8& 7& 4& {\bf 0} &0& & .84& 1.6& 3.5&{\bf 7.5} & 16.8\\
%&$B=30$&4.0& .24&.19 &.17 &.16& & 7& 5& 3&0 &0& & 1.3& 2.5& 5.4&11.6& 26.2 \\
%&$B=40$&.30& .22& {\bf.18}&.17& .16 && 7& 4& {\bf 0}&0& 0 && 1.8& 3.3& {\bf 7.3}&15.5& 33.5 \\
%& Exact& .16& &  & & & & 0& & && & & \multicolumn{5}{l}{271.8}\\
%\hline
\end{tabular}
}
\label{tab_synthtensor_acc_abridged}
%%\vspace*{-0.2cm}
\end{table*}

%\begin{figure*}[t]
%\centering
%\includegraphics[width=4.5cm]{figures/vdiff_rbs_checkerboard.pdf}
%\includegraphics[width=4.5cm]{figures/vdiff_rbs.pdf}
%\includegraphics[width=4.5cm]{figures/rbp_veridim.pdf}
%\caption{Left: Unsorted vector recovery error.
%White indicates good recovery and black indicates poor recovery.
%Middle: Sorted vector recovery error compared to exact decomposition methods.
%Right: Maximum tensor dimension of which fast tensor decomposition is accurate.
%Logarithms on both axes are with respect to 2.
%\yining{Need to be re-done on full-rank tensors}}
%\label{fig_vdiff_veridim}
%%\vspace{-0.3cm}
%\end{figure*}

We demonstrate the effectiveness and efficiency of our proposed sketch based tensor power method on both synthetic tensors and real-world topic modeling problems.
Experimental results involving the fast ALS method are presented in Appendix \ref{appsec:fast_als_experiment}.
All methods are implemented in C++ and tested on a single machine with 8 Intel X5550@2.67Ghz CPUs and 32GB memory.
For synthetic tensor decomposition we use only a single thread; for fast spectral LDA 8 to 16 threads are used.
%The whitening trick is not applied for synthetic tensors since we do not have access to a second-order moment matrix;
%\footnote{Recently \cite{tensor-decomposition-trivial} proposed a simple whitening-based algorithm for noiseless symmetric tensor decomposition
%without additional information. It is unclear whether the proposed method generalizes to noisy inputs.}
%For topic modeling, we do apply whitening and hence operate on the smaller $k\times k\times k$ tensor.

%\vspace{-0.2cm}
\subsection{Synthetic tensors}\label{subsec:synthetic}
%\vspace{-0.2cm}

%We first test the approximation error and running time of fast robust tensor power method (Algorithm \ref{alg_fast_rbp}) and fast ALS (Algorithm \ref{alg_fast_als}) under various $B$ and $b$ settings.
{In Table \ref{tab_synthtensor_acc} we compare our proposed algorithms with exact decomposition methods on synthetic tensors.
Let $n=1000$ be the dimension of the input tensor.
We first generate a random \emph{orthonormal} basis $\{\vct v_i\}_{i=1}^n$ and then set the input tensor $\mat T$
as $\mat T=\mathrm{normalize}(\sum_{i=1}^n{\lambda_i\vct v_i^{\otimes 3}})+\mat E$,
where the eigenvalues $\lambda_i$ satisfy $\lambda_i = 1/i$.
The normalization step makes $\|\mat T\|_F^2=1$ before imposing noise.
The Gaussian noise matrix $\mat E$ is symmetric with $\mat E_{ijk}\sim\nml(0, \sigma/n^{1.5})$
for $i\leq j\leq k$ and noise-to-signal level $\sigma$.
Due to time constraints, we only compare the recovery error and running time on the top 10 recovered eigenvectors of the full-rank input tensor $\mat T$.
%The input tensor $\mat T$ can be written as
%$\mat T=\text{normalize}(\sum_{i=1}^{k}{\lambda_i\vct v_i^{\otimes 3}}) + \mat E$,
%where $\{\vct v_i\}_{i=1}^{k}\subseteq\mathbb R^d$ is a randomly generated orthonormal basis
%and $\lambda_i$ is set to $1-\frac{i}{k}$ for identifiability.
%The normalization step makes $\|\mat T\|_F^2=1$ before imposing noise.
%We set dimension $d=1000$ and intrinsic rank $k=10$.
%The Gaussian noise $\mat E$ satisfies $\|\mat E\|_F^2=0.01$.
%After generating the exact low-rank tensor, we normalize it so that it has unit Frobenius norm and then impose a random Gaussian noise $\mat E$
%with $\|\mat E\|_F^2=0.01$.
%Exact ALS is done by calling the $\mathtt{cp\_als}$ routine in Matlab tensor toolbox.
Both $L$ and $T$ are set to 30. %and for ALS the number of iterations $T$ is set to 30.
Table \ref{tab_synthtensor_acc_abridged} shows that our proposed algorithms achieve reasonable approximation error within a few minutes,
which is much faster then exact methods.
A complete version (Table \ref{tab_synthtensor_acc}) is deferred to Appendix \ref{appsec:supp_experiment}.
\subsection{Topic modeling}
%\vspace{-0.2cm}

%\yining{Emphasize here again that we apply whitening for this task.}

\begin{figure*}
\centering
\includegraphics[width=6cm,height=3.5cm]{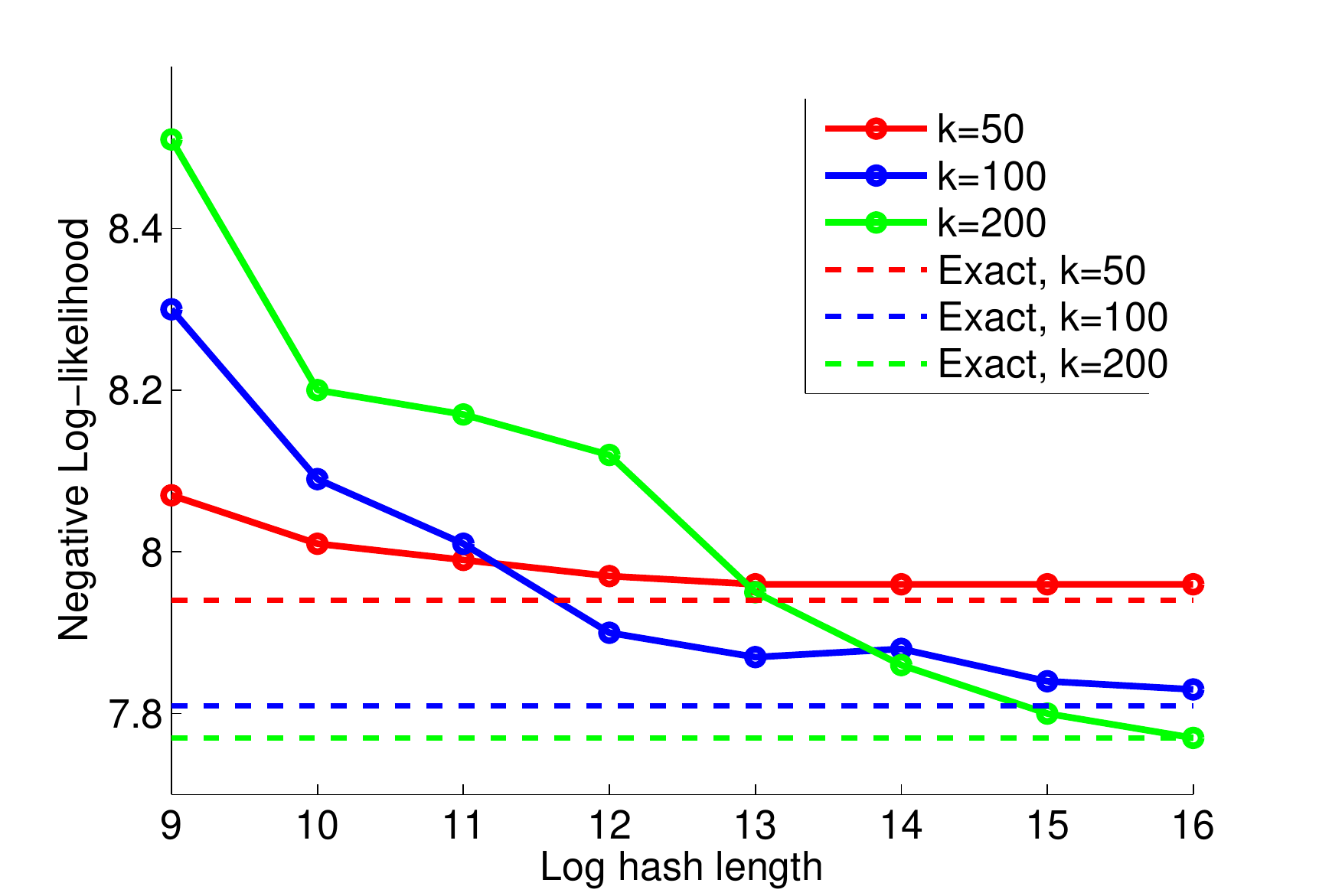}
%\captionof{figure}{Neg. log-likelihood for fast and exact tensor power method on Wikipedia dataset.}
\includegraphics[width=7cm,height=3.5cm]{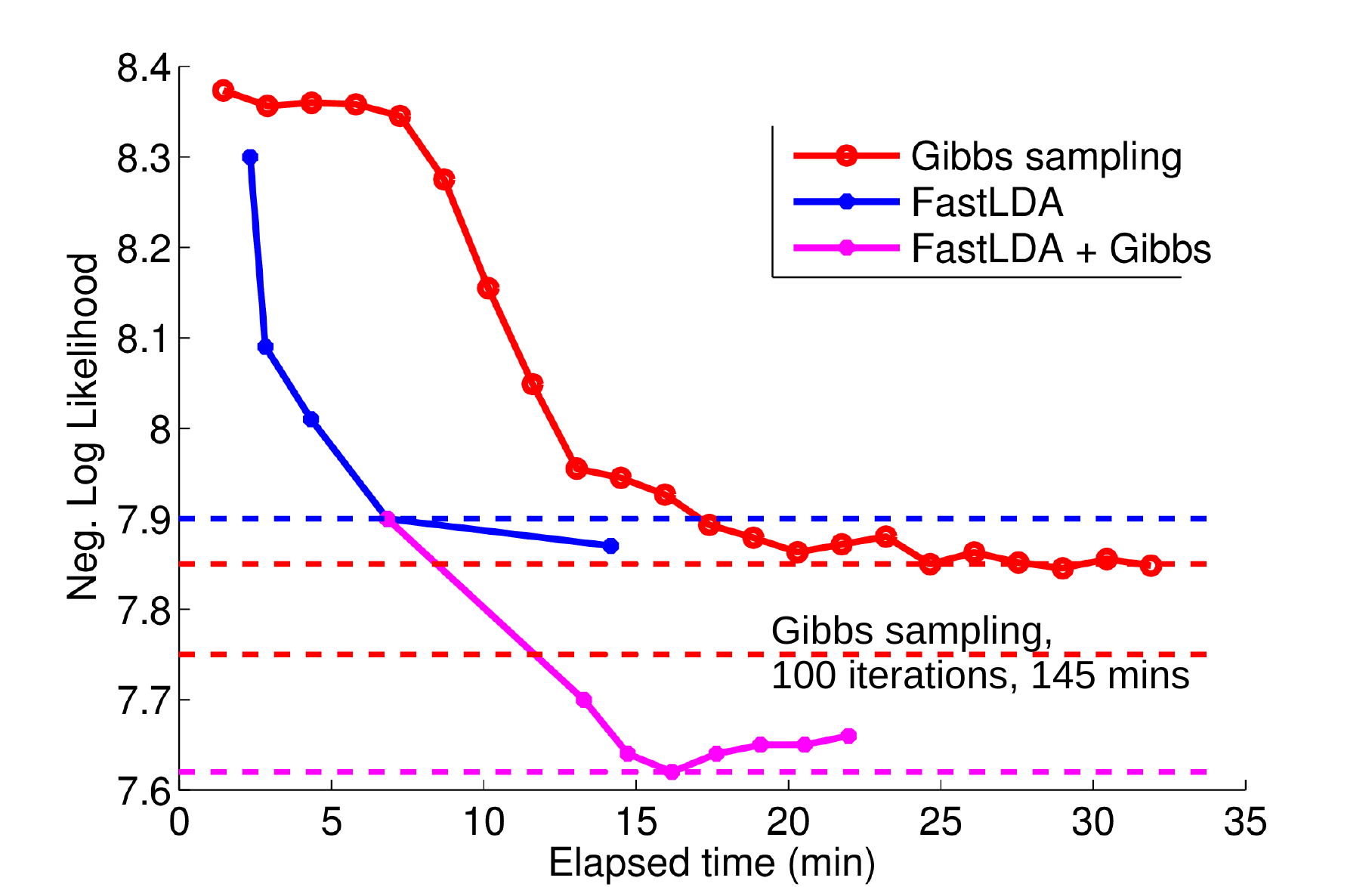}
%\captionof{figure}{Collapsed Gibbs sampling, fast LDA and Gibbs sampling using fast LDA as initialization on the same dataset.}
\captionof{figure}{Left: negative log-likelihood for fast and exact tensor power method on Wikipedia dataset. 
Right: negative log-likelihood for collapsed Gibbs sampling, fast LDA and Gibbs sampling using Fast LDA as initialization.}
\label{fig_errorplot}
\vspace{-0.1cm}
\end{figure*}

\begin{table*}[t]
\centering
\vspace{-0.2cm}
\captionof{table}{Negative log-likelihood and running time (min) on the \emph{large} Wikipedia dataset for 200 and 300 topics.}
\scalebox{0.9}{
\begin{tabular}{cc|cccc|ccccc}
\hline
$k$& & like.& time& $\log_2b$& iters& $k$& like.& time& $\log_2 b$& iters\\
\hline
\multirow{3}{*}{\rotatebox{90}{200}}
&Spectral& 7.49& {\bf 34}& 12& -& \multirow{3}{*}{\rotatebox{90}{300}}& 7.39& {\bf 56}& 13& -\\
&Gibbs& 6.85& 561& -& 30& & 6.38& 818& -& 30\\
&Hybrid&{\bf 6.77} & 144& 12&5  & & {\bf 6.31}& 352& 13& 10\\
\hline
\end{tabular}
}
\label{tab_large_wiki}
%\vspace{-0.2cm}
\end{table*}

We implement a fast spectral inference algorithm for Latent Dirichlet Allocation (LDA \cite{lda})
by combining tensor sketching with existing whitening technique for dimensionality reduction.
Implementation details are provided in Appendix \ref{appsec:lda}.
We compare our proposed fast spectral LDA algorithm with baseline spectral methods and collapsed Gibbs sampling (using GibbsLDA++ \cite{gibbsldapp} implementation)
on two real-world datasets: Wikipedia and Enron.
Dataset details are presented in \ref{appsec:supp_experiment}
Only the most frequent $V$ words are kept and the vocabulary size $V$ is set to 10000.
For the robust tensor power method the parameters are set to $L=50$ and $T=30$.
For ALS we iterate until convergence, or a maximum number of 1000 iterations is reached.
$\alpha_0$ is set to 1.0 and $B$ is set to 30.
%24 independent sketches are used for fast spectral LDA.

Obtained topic models $\mat\Phi\in\mathbb R^{V\times K}$ are evaluated on a held-out dataset consisting of 1000 documents randomly picked out from training datasets.
For each testing document $d$, we fit a topic mixing vector $\hat{\vct\pi}_d\in\mathbb R^K$ by solving the following optimization problem:
$\hat{\vct\pi}_d = \argmin_{\|\vct\pi\|_1 = 1, \vct\pi\geq\vct 0}{\|\vct w_d-\mat\Phi\vct\pi\|_2},$
where $\vct w_d$ is the empirical word distribution of document $d$.
The per-document log-likelihood is then defined as
$\mathcal L_d = \frac{1}{n_d}\sum_{i=1}^{n_d}{\ln p(w_{di})}$,
where $p(w_{di}) = \sum_{k=1}^K{\hat{\vct\pi}_k\mat\Phi_{w_{di},k}}$.
Finally, the average $\mathcal L_d$ over all testing documents is reported.

Figure \ref{fig_errorplot} left shows the held-out negative log-likelihood for fast spectral LDA under different hash lengths $b$.
We can see that as $b$ increases, the performance approaches the exact tensor power method because sketching approximation becomes more accurate.
On the other hand, Table \ref{tab_err_compare} shows that fast spectral LDA runs much faster than exact tensor decomposition methods
while achieving comparable performance on both datasets.

{
Figure \ref{fig_errorplot} right compares the convergence of collapsed Gibbs sampling with different number of iterations and fast spectral LDA with different hash lengths on Wikipedia dataset.
For collapsed Gibbs sampling, we set $\alpha = 50/K$ and $\beta = 0.1$ following \cite{collapsed-gibbs-lda}.
As shown in the figure, fast spectral LDA achieves comparable held-out likelihood while running faster than collapsed Gibbs sampling.
We further take the dictionary $\mat\Phi$ output by fast spectral LDA and use it as initializations for collapsed Gibbs sampling
(the word topic assignments $\vct z$ are obtained by 5-iteration Gibbs sampling, with the dictionary $\mat\Phi$ fixed).
The resulting Gibbs sampler converges much faster:
with only 3 iterations it already performs much better than a randomly initialized Gibbs sampler run for 100 iterations,
which takes 10x more running time.

We also report performance of fast spectral LDA and collapsed Gibbs sampling on a larger dataset in Table \ref{tab_large_wiki}.
The dataset was built by crawling 1,085,768 random Wikipedia pages and a held-out evaluation set was built by randomly picking out 1000 documents from the dataset.
Number of topics $k$ is set to 200 or 300, and after getting topic dictionary $\mat\Phi$ from fast spectral LDA
we use 2-iteration Gibbs sampling to obtain word topic assignments $\vct z$.
Table \ref{tab_large_wiki} shows that the hybrid method (i.e., collapsed Gibbs sampling initialized by spectral LDA)
achieves the best likelihood performance in a much shorter time, compared to a randomly initialized Gibbs sampler.
}

\vspace{-0.15in}
\section{Conclusion}
\vspace{-0.15in}
In this work we proposed a sketching based approach to efficiently compute tensor CP decomposition with provable guarantees.
We apply our proposed algorithm on learning latent topics of unlabeled document collections
and achieve significant speed-up compared to vanilla spectral and collapsed Gibbs sampling methods.
Some interesting future directions include further improving the sample complexity analysis and applying the framework
to a broader class of graphical models.

\paragraph{Acknowledgement:} Anima Anandkumar is supported in part by the Microsoft Faculty Fellowship and the Sloan Foundation.
Alex Smola is supported in part by a Google Faculty Research Grant.

{\small
\bibliographystyle{IEEE}
\bibliography{fftlda}

\begin{thebibliography}{1}

\bibitem{tensor-power-method}
A.~Anandkumar, R.~Ge, D.~Hsu, S.~Kakade, and M.~Telgarsky.
\newblock Tensor decompositions for learning latent variable models.
\newblock {\em Journal of Machine Learning Research}, 15:2773--2832, 2014.

\bibitem{tensor-toolbox-1}
B.~Bader, T.~Kolda, et~al.
\newblock {MATLAB} tensor toolbox version 2.5.
\newblock Available online, 2012.

\bibitem{tensor-toolbox-2}
B.~W. Bader and T.~G. Kolda.
\newblock Algorithm 862: {MATLAB} tensor classes for fast algorithm
  prototyping.
\newblock {\em ACM Transactions on Mathematical Software}, 32(4):635--653,
  2006.

\bibitem{lda}
D.~M. Blei, A.~Y. Ng, and M.~I. Jordan.
\newblock Latent dirichlet allocation.
\newblock {\em Journal of machine Learning research}, 3:993--1022, 2003.

\bibitem{noisy-tensor-power-method}
M.~Hardt and E.~Price.
\newblock The noisy power method: A meta algorithm with applications.
\newblock In {\em NIPS}, 2014.

\bibitem{tensor-review}
T.~Kolda and B.~Bader.
\newblock Tensor decompositions and applications.
\newblock {\em SIAM Review}, 51(3):455--500, 2009.

\bibitem{scalable-spectral-lda}
C.~Wang, X.~Liu, Y.~Song, and J.~Han.
\newblock Scalable moment-based inference for latent dirichlet allocation.
\newblock In {\em ECML/PKDD}, 2014.

\bibitem{spectral-slda}
Y.~Wang and J.~Zhu.
\newblock Spectral methods for supervised topic models.
\newblock In {\em NIPS}, 2014.

\end{thebibliography}


\begin{thebibliography}{10}

\bibitem{tensor-power-method}
A.~Anandkumar, R.~Ge, D.~Hsu, S.~Kakade, and M.~Telgarsky.
\newblock Tensor decompositions for learning latent variable models.
\newblock {\em Journal of Machine Learning Research}, 15:2773--2832, 2014.

\bibitem{fast-tensor-sampling}
S.~Bhojanapalli and S.~Sanghavi.
\newblock A new sampling technique for tensors.
\newblock {\em arXiv:1502.05023}, 2015.

\bibitem{lda}
D.~M. Blei, A.~Y. Ng, and M.~I. Jordan.
\newblock Latent dirichlet allocation.
\newblock {\em Journal of machine Learning research}, 3:993--1022, 2003.

\bibitem{cp-knowledge-base}
A.~Carlson, J.~Betteridge, B.~Kisiel, B.~Settles, E.~R. Hruschka~Jr, and T.~M.
  Mitchell.
\newblock Toward an architecture for never-ending language learning.
\newblock In {\em AAAI}, 2010.

\bibitem{cp-ref2}
J.~D. Carroll and J.-J. Chang.
\newblock Analysis of individual differences in multidimensional scaling via an
  n-way generalization of ``eckart-young” decomposition.
\newblock {\em Psychometrika}, 35(3):283--319, 1970.

\bibitem{spectral-graphical-model}
A.~Chaganty and P.~Liang.
\newblock Estimating latent-variable graphical models using moments and
  likelihoods.
\newblock In {\em ICML}, 2014.

\bibitem{count-sketch}
M.~Charikar, K.~Chen, and M.~Farach-Colton.
\newblock Finding frequent items in data streams.
\newblock {\em Theoretical Computer Science}, 312(1):3--15, 2004.

\bibitem{distributed-als}
J.~H. Choi and S.~Vishwanathan.
\newblock {DFacTo}: Distributed factorization of tensors.
\newblock In {\em NIPS}, 2014.

\bibitem{privacy-adaptive}
C.~Dwork, V.~Feldman, M.~Hardt, T.~Pitassi, O.~Reingold, and A.~Roth.
\newblock Preserving statistical validity in adaptive data analysis.
\newblock In {\em STOC}, 2015.

\bibitem{cp-topographic}
A.~S. Field and D.~Graupe.
\newblock Topographic component (parallel factor) analysis of multichannel
  evoked potentials: practical issues in trilinear spatiotemporal
  decomposition.
\newblock {\em Brain Topography}, 3(4):407--423, 1991.

\bibitem{collapsed-gibbs-lda}
T.~L. Griffiths and M.~Steyvers.
\newblock Finding scientific topics.
\newblock {\em Proceedings of the National Academy of Sciences}, 101(suppl
  1):5228--5235, 2004.

\bibitem{cp-ref1}
R.~A. Harshman.
\newblock Foundations of the {PARAFAC} procedure: Models and conditions for an
  explanatory multi-modal factor analysis.
\newblock {\em UCLA Working Papers in Phonetics}, 16:1--84, 1970.

\bibitem{anima-distributed-lda}
F.~Huang, S.~Matusevych, A.~Anandkumar, N.~Karampatziakis, and P.~Mineiro.
\newblock Distributed latent dirichlet allocation via tensor factorization.
\newblock In {\em NIPS Optimization Workshop}, 2014.

\bibitem{sgd-spec}
F.~Huang, U.~N. Niranjan, M.~U. Hakeem, and A.~Anandkumar.
\newblock Fast detection of overlapping communities via online tensor methods.
\newblock {\em arXiv:1309.0787}, 2013.

\bibitem{khartri-rao-fast-property}
A.~Jain.
\newblock Fundamentals of digital image processing, 1989.

\bibitem{gigatensor}
U.~Kang, E.~Papalexakis, A.~Harpale, and C.~Faloutsos.
\newblock Gigatensor: Scaling tensor analysis up by 100 times - algorithms and
  discoveries.
\newblock In {\em KDD}, 2012.

\bibitem{enron}
B.~Klimt and Y.~Yang.
\newblock Introducing the enron corpus.
\newblock In {\em CEAS}, 2004.

\bibitem{cp-dblp}
T.~Kolda and B.~Bader.
\newblock The tophits model for higher-order web link analysis.
\newblock In {\em Workshop on link analysis, counterterrorism and security},
  2006.

\bibitem{tensor-review}
T.~Kolda and B.~Bader.
\newblock Tensor decompositions and applications.
\newblock {\em SIAM Review}, 51(3):455--500, 2009.

\bibitem{cp-webgraph}
T.~G. Kolda and J.~Sun.
\newblock Scalable tensor decompositions for multi-aspect data mining.
\newblock In {\em ICDM}, 2008.

\bibitem{cp-eeg}
M.~M{\o}rup, L.~K. Hansen, C.~S. Herrmann, J.~Parnas, and S.~M. Arnfred.
\newblock Parallel factor analysis as an exploratory tool for wavelet
  transformed event-related eeg.
\newblock {\em NeuroImage}, 29(3):938--947, 2006.

\bibitem{compressed-matrix-multiplication}
R.~Pagh.
\newblock Compressed matrix multiplication.
\newblock In {\em ITCS}, 2012.

\bibitem{cs-poly-kernel}
N.~Pham and R.~Pagh.
\newblock Fast and scalable polynomial kernels via explicit feature maps.
\newblock In {\em KDD}, 2013.

\bibitem{fast-als}
A.-H. Phan, P.~Tichavsky, and A.~Cichocki.
\newblock Fast alternating {LS} algorithms for high order {CANDECOMP/PARAFAC}
  tensor factorizations.
\newblock {\em IEEE Transactions on Signal Processing}, 61(19):4834--4846,
  2013.

\bibitem{gibbsldapp}
X.-H. Phan and C.-T. Nguyen.
\newblock {GibbsLDA++}: A {C/C++} implementation of latent dirichlet allocation
  (lda), 2007.

\bibitem{universal-hash-2}
M.~Pǎtra{\c{s}}cu and M.~Thorup.
\newblock The power of simple tabulation hashing.
\newblock {\em Journal of the ACM}, 59(3):14, 2012.

\bibitem{mach}
C.~Tsourakakis.
\newblock {MACH}: Fast randomized tensor decompositions.
\newblock In {\em SDM}, 2010.

\bibitem{spectral-ibp}
H.-Y. Tung and A.~Smola.
\newblock Spectral methods for indian buffet process inference.
\newblock In {\em NIPS}, 2014.

\bibitem{scalable-spectral-lda}
C.~Wang, X.~Liu, Y.~Song, and J.~Han.
\newblock Scalable moment-based inference for latent dirichlet allocation.
\newblock In {\em ECML/PKDD}, 2014.

\bibitem{spectral-slda}
Y.~Wang and J.~Zhu.
\newblock Spectral methods for supervised topic models.
\newblock In {\em NIPS}, 2014.

\end{thebibliography}
}

\clearpage

\begin{appendices}

\section{Supplementary experimental results}\label{appsec:supp_experiment}

The Wikipedia dataset is built by crawling all documents in all subcategories within 3 layers below the \emph{science} category.
The Enron dataset is from the Enron email corpus \cite{enron}.
After usual cleaning steps, the Wikipedia dataset has $114,274$ documents with an average $512$ words per document;
the Enron dataset has $186,501$ emails with average $91$ words per email.

\begin{table*}[h]
%\begin{threeparttable}[t]
\centering
\caption{Squared residual norm on top 10 recovered eigenvectors of 1000d tensors %$\|\mat T-\sum_{i=1}^k{\lambda_i\vct v_i^{\otimes 3}}\|_F^2$
and running time (excluding I/O and sketch building time) for plan (exact) and sketched robust tensor power methods.
Two vectors are considered mismatched (wrong) if $\|\vct v-\hat{\vct v}\|_2^2 > 0.1$.
%``RB" stands for robust tensor power method.
}
%\vskip 0.05in
\scalebox{0.85}{
\begin{tabular}{ccccccccccccccccccc}
\hline
%\abovespace
 & & \multicolumn{5}{c}{Residual norm}& &\multicolumn{5}{c}{No. of wrong vectors}& &\multicolumn{5}{c}{Running time (min.)}\\
\cline{3-7}\cline{9-13}\cline{15-19}
&  $\log_2(b)$:&  12& 13& 14& 15&16&  & 12& 13& 14& 15&16&  &  12& 13& 14& 15& 16\\
\hline
\multirow{4}{*}{\rotatebox{90}{$\sigma=.01$}}
&$B=20$&.40&.19 &.10 &{\bf .09} &.08& & 8& 6& 3&{\bf 0} &0& & .85& 1.6& 3.5&{\bf 7.4}& 16.6 \\
&$B=30$&.26& .10&.09 &.08 &.07& & 7& 5& 2&0 &0& & 1.3& 2.4& 5.3&11.3& 24.6 \\
&$B=40$&.17& .10&{\bf .08} &.08 &.07 && 7& 4&{\bf 0}&0 &0& & 1.8& 3.3& {\bf 7.3}&15.2& 33.0 \\
& Exact& .07& & & & & & 0& & & & & & \multicolumn{5}{l}{293.5}\\
\hline
\multirow{4}{*}{\rotatebox{90}{$\sigma=.1$}}
&$B=20$&.52&3.1 &.21 &{\bf.18} &.17& & 8& 7& 4& {\bf 0} &0& & .84& 1.6& 3.5&{\bf 7.5} & 16.8\\
&$B=30$&4.0& .24&.19 &.17 &.16& & 7& 5& 3&0 &0& & 1.3& 2.5& 5.4&11.6& 26.2 \\
&$B=40$&.30& .22& {\bf.18}&.17& .16 && 7& 4& {\bf 0}&0& 0 && 1.8& 3.3& {\bf 7.3}&15.5& 33.5 \\
& Exact& .16& &  & & & & 0& & && & & \multicolumn{5}{l}{271.8}\\
\hline
\end{tabular}
}
\label{tab_synthtensor_acc}
%\vspace*{-0.2cm}
\end{table*}

\begin{table}[h]
\centering
\caption{Selected negative log-likelihood and running time (min) for fast and exact spectral methods on Wikipedia (top) and Enron (bottom) datasets.}
%\vskip 0.05in
\scalebox{0.9}{
\begin{tabular}{cccccccccccccccc}
\hline
& & \multicolumn{3}{c}{$k=50$}& & \multicolumn{3}{c}{$k=100$}& & \multicolumn{3}{c}{$k=200$}\\
\cline{3-5}\cline{7-9} \cline{11-13}
& & Fast RB& RB& ALS& & Fast RB& RB& ALS&&  Fast RB& RB& ALS\\
\hline
\multirow{3}{*}{\rotatebox{90}{Wiki.}}
& like.& 8.01& {\bf 7.94}& 8.16&& 7.90& {\bf 7.81}& 7.93&& 7.86& {\bf 7.77}& 7.89\\
& time& {\bf 2.2}& 97.7& 2.4& & {\bf 6.8}& 135& 29.3& & {\bf 57.3}& 423& 677\\
& $\log_2 b$& 10& -& -& & 12& -& -& & 14& -& - \\
\hline
\multirow{3}{*}{\rotatebox{90}{Enron}}
& like.& 8.31& {\bf 8.28}& 8.22& & 8.18& {\bf 8.09}& 8.30& & 8.26& {\bf 8.18}& 8.27\\
& time& {\bf 2.4}& 45.8& 5.2& & {\bf 3.7}& 93.9& 40.6& & {\bf 6.4}& 219& 660\\
& $\log_2 b$& 11& -& -& & 11& -& -& & 11& -& -\\
\hline
\end{tabular}
}
\label{tab_err_compare}
\end{table}

\section{Fast tensor power method via symmetric sketching}\label{appsec:detail_tensor_power}

In this section we show how to do fast tensor power method using symmetric tensor sketches.
More specifically, we explain how to approximately compute $\mat T(\vct u,\vct u,\vct u)$ and $\mat T(\mat I,\vct u,\vct u)$ when
colliding hashes are used.

For symmetric tensors $\mat A$ and $\mat B$, their inner product can be approximated by
\begin{equation}
\langle\mat A,\mat B\rangle \approx \langle\tilde{\vct s}_{\mat A},\tilde{\vct s}_{\widetilde{\mat B}}\rangle,
\label{eq_ab_collide}
\end{equation}
where $\widetilde{\mat B}$ is an ``upper-triangular" tensor defined as
\begin{equation}
\widetilde{\mat B}_{i,j,k} = \left\{\begin{array}{ll}
\mat B_{i,j,k},& \text{if }i\leq j\leq k;\\
0,& \text{otherwise}.\end{array}\right.
\end{equation}
Note that in Eq. (\ref{eq_ab_collide}) only the matrix $\mat B$ is ``truncated".
We show this gives consistent estimates of $\langle\mat A,\mat B\rangle$ in Appendix \ref{appsec:tensor_sketch_bound}.

Recall that $\mat T(\vct u,\vct u,\vct u)=\langle\mat T,\mat X\rangle$ where $\mat X=\vct u\otimes\vct u\otimes\vct u$.
The symmetric tensor sketch $\tilde{\vct s}_{\widetilde{\mat X}}$ can be computed as
\begin{equation}
\tilde{\vct s}_{\widetilde{\mat X}} = \frac{1}{6}\tilde{\vct s}_{\vct u}^{\otimes 3} + \frac{1}{2}\tilde{\vct{s}}_{2,\vct u\circ\vct u}*\tilde{\vct s}_{\vct u} + \frac{1}{3}\tilde{\vct{s}}_{3,\vct u\circ\vct u\circ\vct u},
\end{equation}
where $\tilde{s}_{2,\vct u\circ\vct u}(t) = \sum_{2h(i)=t}{\sigma(i)^2\vct u_i^2}$ and
$\tilde{s}_{3,\vct u\circ\vct u\circ\vct u}(t) = \sum_{3h(i)=t}{\sigma(i)^3\vct u_i^3}$.
As a result,
\begin{equation}
\mat T(\vct u,\vct u,\vct u) \approx
\frac{1}{6}\langle\mathcal F(\tilde{\vct s}_{\mat T}), \mathcal F(\tilde{\vct s}_{\vct u})\circ\mathcal F(\tilde{\vct s}_{\vct u})\circ\mathcal F(\tilde{\vct s}_{\vct u} )\rangle\\
+ \frac{1}{2}\langle\mathcal F(\tilde{\vct s}_{\mat T}), \mathcal F(\tilde{\vct s}_{2,\vct u\circ\vct u})\circ\mathcal F(\tilde{\vct s}_{\vct u})\rangle
+ \frac{1}{3}\langle\tilde{\vct s}_{\mat T}, \tilde{\vct s}_{3,\vct u\circ\vct u\circ\vct u}\rangle.
\label{eq_tuuu_fast}
\end{equation}

For $\mat T(\mat I,\vct u,\vct u)$ recall that $[\mat T(\mat I,\vct u,\vct u)]_i=\langle\mat T,\mat Y_i\rangle$
where $\mat Y_i=\vct e_i\otimes\vct u\otimes\vct u$.
We first symmetrize it by defining $\mat Z_i=\vct e_i\otimes\vct u\otimes\vct u+\vct u\otimes\vct e_i\otimes\vct u+\vct u\otimes\vct u\otimes\vct e_i$.
\footnote{As long as $\mat A$ is symmetric, we have $\langle\mat A,\mat Y_i\rangle = \langle\mat A,\mat Z_i\rangle/3$.}
The sketch of $\widetilde{\mat Z}_i$ can be subsequently computed as
\begin{equation}
\tilde{\vct s}_{\widetilde{\mat Z}_i} = \frac{1}{2}\tilde{\vct s}_{\vct u}*\tilde{\vct s}_{\vct u}*\tilde{\vct s}_{\vct e_i}
+ \frac{1}{2}\tilde{\vct s}_{2,\vct u\circ\vct u}*\tilde{\vct s}_{\vct e_i}
+ \tilde{\vct s}_{2,\vct e_i\circ\vct u}*\tilde{\vct s}_{\vct u}
+ \tilde{\vct s}_{3,\vct e_i\circ\vct u\circ\vct u}.
\end{equation}
Consequently,
\begin{multline}
\mat T(\mat I,\vct u,\vct u)
\approx \left\langle\mathcal F^{-1}\left(\mathcal F(\tilde{\vct s}_{\mat T})\circ\overline{\mathcal F(\tilde{\vct s}_{\vct u})}\right), \tilde{\vct s}_{2,\vct e_i\circ\vct u}\right\rangle
+\frac{1}{6}\left\langle\mathcal F^{-1}\left(\mathcal F(\tilde{\vct s}_{\mat T})\circ\overline{\mathcal F(\tilde{\vct s}_{\vct u})}\circ\overline{\mathcal F(\tilde{\vct s}_{\vct u})} \right), \tilde{\vct s}_{\vct e_i}\right\rangle\\
+ \frac{1}{6}\left\langle\mathcal F^{-1}\left(\mathcal F(\tilde{\vct s}_{\mat T})\circ\overline{\mathcal F(\tilde{\vct s}_{2,\vct u\circ\vct u})}\right),
\tilde{\vct s}_{\vct e_i}\right\rangle
+ \langle\tilde{\vct s}_{\mat T}, \tilde{\vct s}_{3, \vct e_i\circ\vct u\circ\vct u}\rangle.
\label{eq_tiuu_fast}
\end{multline}
Note that all of $\tilde{\vct s}_{\vct e_i}$, $\tilde{\vct s}_{2,\vct e_i\circ\vct u}$ and $\tilde{\vct s}_{3,\vct e_i\circ\vct u\circ\vct u}$ have
exactly one nonzero entries.
So we can pre-compute all terms on the left sides of inner products in Eq. (\ref{eq_tiuu_fast}) and then read off the values for each entry in $\mat T(\mat I,\vct u,\vct u)$.

\section{Fast ALS: method and simulation result}\label{appsec:fast_als}

In this section we describe how to use tensor sketching to accelerate the Alternating Least Squares (ALS) method for tensor CP decomposition.
We also provide experimental results on synthetic data and compare our fast ALS implementation with the Matlab tensor toolbox \citeapp{tensor-toolbox-1,tensor-toolbox-2},
which is widely considered to be the state-of-the-art for tensor decomposition.

\subsection{Alternating Least Squares}

Alternating Least Squares (ALS) is a popular method for tensor CP decompositions \citeapp{tensor-review}.
The algorithm maintains $\vct\lambda\in\mathbb R^k$, $\mat A,\mat B,\mat C\in\mathbb R^{n\times k}$ and iteratively perform the following update steps:
%\footnote{Only updates for $\mat A$ is shown. Updates for $\mat B,\mat C$ are similar.}
\begin{align}
\widehat{\mat A} &= \mat T_{(1)}(\mat C\odot\mat B)(\mat C^\top\mat C\circ \mat B^\top\mat B)^\dagger.\\
\widehat{\mat B} &= \mat T_{(1)}(\widehat{\mat A}\odot\mat C)(\widehat{\mat A}^\top\widehat{\mat A}\circ \mat C^\top\mat C)^\dagger;\nonumber\\
\widehat{\mat C} &= \mat T_{(1)}(\widehat{\mat B}\odot\widehat{\mat A})(\widehat{\mat B}^\top\widehat{\mat B}\circ \widehat{\mat A}^\top\widehat{\mat A})^\dagger.\nonumber
\end{align}
After each update, $\hat\lambda_r$ is set to $\|\vct a_r\|_2$ (or $\|\vct b_r\|_2, \|\vct c_r\|_2$) for $r=1,\cdots,k$
and the matrix $\mat A$ (or $\mat B,\mat C$) is normalized so that each column has unit norm.
The final low-rank approximation is obtained by $\sum_{i=1}^k{\hat\lambda_i\hat{\vct a}_i\otimes\hat{\vct b}_i\otimes\hat{\vct c}_i}$.

There is no guarantee that ALS converges or gives a good tensor decomposition.
Nevertheless, it works reasonably well in most applications \citeapp{tensor-review}.
In general ALS requires $O(T(n^3k+k^3))$ computations and $O(n^3)$ storage,
where $T$ is the number of iterations.

\subsection{Accelerated ALS via sketching}

\begin{algorithm}[t]
\caption{Fast ALS method}
\begin{algorithmic}[1]
\State \textbf{Input}: $\mat T\in\mathbb R^{n\times n\times n}$, target rank $k$, $T$, $B$, $b$.
\State \textbf{Initialize}: $B$ independent index hash functions $h^{(1)},\cdots, h^{(B)}$ and $\sigma^{(1)},\cdots,\sigma^{(B)}$;
random matrices $\mat A,\mat B,\mat C\in\mathbb R^{n\times k}$; $\{\lambda_i\}_{i=1}^k$.
\State For $m=1,\cdots, B$ compute ${\vct s}^{(m)}_{{\mat T}}\in\mathbb C^b$.
\For{$t=1$ to $T$}
    \State Compute count sketches $\vct s_{\vct b_i}$, $\vct s_{\vct c_i}$ for $i=1,\cdots,k$.
    For each $i=1,\cdots,k;m=1,\cdots,b$ compute $\vct v_i^{(m)}\approx\mat T(\mat I,\vct b_i,\vct c_i)$.
    \State $\bar{\vct v}_{ij} \gets \median(\Re(\vct v^{(1)}_{ij}), \Re(\vct v^{(2)}_{ij}), \cdots, \Re(\vct v^{(B)}_{ij}))$.
    \State Set $\widehat{\mat A}=\{\bar{\vct v}\}_{ij}$ and $\hat\lambda_i=\|\hat{\vct a}_i\|$; afterwards, normalize each column of $\mat A$.
    \State Update $\mat B$ and $\mat C$ similarly.
\EndFor
\State \textbf{Output}: eigenvalues $\{\lambda_i\}_{i=1}^k$; solutions $\mat A,\mat B,\mat C$.
\end{algorithmic}
\label{alg_fast_als}
\end{algorithm}

Similar to robust tensor power method, the ALS algorithm can be significantly accelerated by using the idea of sketching as shown in this work.
However, for ALS we cannot use colliding hashes because though the input tensor $\mat T$ is symmetric,
its CP decomposition is not since we maintain three different solution matrices $\mat A,\mat B$ and $\mat C$.
As a result, we roll back to asymmetric tensor sketches defined in Eq. (\ref{eq_asym_tensor_sketch}).
Recall that given $\mat A,\mat B,\mat C\in\mathbb R^{n\times k}$ we want to compute
\begin{equation}
\hat{\mat A} = \mat T_{(1)}(\mat C\odot\mat B)(\mat C^\top\mat C\circ\mat B^\top\mat B)^\dagger.
\label{eq_symmetric_als}
\end{equation}
When $k$ is much smaller than the ambient tensor dimension $n$
the computational bottleneck of Eq. (\ref{eq_symmetric_als}) is $\mat T_{(1)}(\mat C\odot\mat B)$,
which requires $O(n^3k)$ operations.
Below we show how to use sketching to speed up this computation.

Let $\vct x\in\mathbb R^{n^2}$ be one row in $\mat T_{(1)}$ and consider $(\mat C\odot\mat B)^\top\vct x$.
It can be shown that \cite{khartri-rao-fast-property}
\begin{equation}
\left[(\mat C\odot\mat B)^\top\vct x\right]_i = \vct b_i^\top\mat X\vct c_i,\quad\forall i=1,\cdots,k,
\label{eq_khartri_rao_fast_property}
\end{equation}
where $\mat X\in\mathbb R^{n\times n}$ is the reshape of vector $\vct x$.
Subsequently, the product $\mat T_{(1)}(\mat C\odot\mat B)$ can be re-written as
\begin{equation}
\mat T_{(1)}(\mat C\odot\mat B) = [\mat T(\mat I,\vct b_1,\vct c_1); \cdots; \mat T(\mat I,\vct b_k,\vct c_k)].
\end{equation}

Using Proposition \ref{prop_tiuu} we can compute each of $\mat T(\mat I,\vct b_i,\vct c_i)$ in $O(n+b\log b)$ iterations.
Note that in general $\vct b_i\neq\vct c_i$, but Proposition \ref{prop_tiuu} still holds by replacing one of the two $\vct s_{\vct u}$ sketches.
As a result, $\mat T_{(1)}(\mat C\odot\mat B)$ can be computed in $O(k(n+b\log b))$ operations once $\vct s_{\mat T}$ is computed.
The pseudocode of fast ALS is listed in Algorithm \ref{alg_fast_als}.
Its time complexity and space complexity are $O(T(k(n+Bb\log b)+k^3))$ (excluding the time for building $\vct s_{\mat T}$)
and $O(Bb)$, respectively.

\subsection{Simulation results}\label{appsec:fast_als_experiment}

\begin{table*}[t]
%\begin{threeparttable}[t]
\centering
\caption{Squared residual norm on top 10 recovered eigenvectors of 1000d tensors %$\|\mat T-\sum_{i=1}^k{\lambda_i\vct v_i^{\otimes 3}}\|_F^2$
and running time (excluding I/O and sketch building time) for plain (exact) and sketched ALS algorithms.
Two vectors are considered mismatched (wrong) if $\|\vct v-\hat{\vct v}\|_2^2 > 0.1$.
%``RB" stands for robust tensor power method.
}
%\vskip 0.05in
\scalebox{0.85}{
\begin{tabular}{ccccccccccccccccccc}
\hline
%\abovespace
 & & \multicolumn{5}{c}{Residual norm}& &\multicolumn{5}{c}{No. of wrong vectors}& &\multicolumn{5}{c}{Running time (min.)}\\
\cline{3-7}\cline{9-13}\cline{15-19}
&  $\log_2(b)$:&  12& 13& 14& 15&16&  & 12& 13& 14& 15&16&  &  12& 13& 14& 15& 16\\
\hline
\multirow{4}{*}{\rotatebox{90}{$\sigma=.01$}}
&$B=20$&.71& .41 &.25 &{ .17} &.12& & 10& 9& 7&{ 6} &4& & .11& .22& .49&{1.1}& 2.4 \\
&$B=30$&.50& .34&.21 &.14 &.11& & 9& 8&7 &5&3& & .17& .33& .75&1.6& 3.5 \\
&$B=40$&.46& .28&{.17} &.10 &{\bf .07} && 9& 8&{ 6}&5 &{\bf 1}& & .23& .45& {1.0}&2.2&  {\bf 4.7}\\
& Exact\textsuperscript{$\dagger$}& .07& & & & & & 1& & & & & & \multicolumn{5}{l}{22.8}\\
\hline
\multirow{4}{*}{\rotatebox{90}{$\sigma=.1$}}
&$B=20$&.88&.50 &.35 &{ .28} &.23& & 10& 8& 7& { 6} &6& & .13& .32& .78&{1.5} & 3.2\\
&$B=30$&.78& .44&.30 &.24 &.21& & 9& 8& 7&5 &6& & .21& .50& 1.1&2.2& 4.7 \\
&$B=40$&.56& .38& {.28}&.19& {\bf .16} && 9& 8& {6}&4& {\bf 2} && .29& .69& {1.5}&3.5& {\bf 6.3} \\
& Exact\textsuperscript{$\dagger$}& .17& &  & & & & 2& & && & & \multicolumn{5}{l}{32.3}\\
\hline
\multicolumn{19}{l}{\textsuperscript{$\dagger$}\footnotesize{Calling $\mathtt{cp\_als}$ in Matlab tensor toolbox. It is run for exactly $T=30$ iterations.}}
\end{tabular}
}
\label{tab_als}
%%\vspace*{-0.5cm}
\end{table*}

We compare the performance of fast ALS with a brute-force implementation under various hash length settings on synthetic datasets in Table \ref{tab_als}.
Settings for generating the synthetic dataset is exactly the same as in Section \ref{subsec:synthetic}.
We use the \texttt{cp\_als} routine in Matlab tensor toolbox as the reference brute-force implementation of ALS.
For fair comparison, exactly $T=30$ iterations are performed for both plain and accelerated ALS algorithms.
Table \ref{tab_als} shows that when sketch length $b$ is not too small, fast ALS achieves comparable accuracy with exact methods
while being much faster in terms of running time.

\section{Spectral LDA and fast spectral LDA}\label{appsec:lda}

Latent Dirichlet Allocation (LDA, \citeapp{lda}) is a powerful tool in topic modeling.
In this section we first review the LDA model and introduce the tensor decomposition method for learning LDA models, which was proposed in \citeapp{tensor-power-method}.
We then provide full details of our proposed fast spectral LDA algorithm.
Pseudocode for fast spectral LDA is listed in Algorithm \ref{alg_fast_lda}.

\subsection{LDA and spectral LDA}

\begin{algorithm}[t]
\caption{Fast spectral LDA}
\begin{algorithmic}[1]
\State \textbf{Input}: Unlabeled documents, $V$, $K$, $\alpha_0$, $B$, $b$.
\State Compute empirical moments $\widehat{\mat M}_1$ and $\widehat{\mat M}_2$ defined in Eq. (\ref{eq_moment1},\ref{eq_moment2}).
\State $[\mat U,\mat S,\mat V]\gets\text{truncatedSVD}(\widehat{\mat M}_2,k)$; $\mat W_{ik}\gets\frac{\mat U_{ik}}{\sqrt{\sigma_k}}$.
\State Build $B$ tensor sketches of $\widehat{\mat M}_3(\mat W,\mat W,\mat W)$.
\State Find CP decomposition $\{\lambda_i\}_{i=1}^k,\mat A=\mat B=\mat C=\{\vct v_i\}_{i=1}^k$ of $\widehat{\mat M}_3(\mat W,\mat W,\mat W)$
using either fast tensor power method or fast ALS method.
\State \textbf{Output}: estimates of prior parameters $\hat\alpha_i=\frac{4\alpha_0(\alpha_0+1)}{(\alpha_0+2)^2\lambda_i^2}$
and topic distributions $\hat{\vct\mu}_i=\frac{\alpha_0+2}{2}\lambda_i(\mat W^\dagger)^\top\vct v_i$.
\end{algorithmic}
\label{alg_fast_lda}
\end{algorithm}

LDA models a collection of documents by a topic dictionary $\mat\Phi\in\mathbb R^{V\times K}$ and a Dirichlet prior $\vct\alpha\in\mathbb R^k$,
where $V$ is the vocabulary size and $k$ is the number of topics.
Each column in $\mat\Phi$ is a probability distribution (i.e., non-negative and sum to one) representing the word distribution of a particular topic.
For each document $d$, a topic mixing vector $\vct h_d\in\mathbb R^k$ is first sampled from a Dirichlet distribution parameterized by $\vct\alpha$.
Afterwards, words in document $d$ i.i.d. sampled from a categorical distribution parameterized by $\mat\Phi\vct h_d$.

A spectral method for LDA based on 3rd-order robust tensor decomposition was proposed in \citeapp{tensor-power-method}
to provably learn LDA model parameters from a polynomial number of training documents.
Let $\vct x\in\mathbb R^V$ represent a single word; that is, for word $w$ we have $x_w=1$ and $x_{w'}=0$ for all $w'\neq w$.
Define first, second and third order moments $\mat M_1,\mat M_2$ and $\mat M_3$ as follows:
\begin{align}
\mat M_1 &= \mathbb E[\vct x_1];\label{eq_moment1}\\
\mat M_2 &= \mathbb E[\vct x_1\otimes\vct x_2] - \frac{\alpha_0}{\alpha_0+1}\mat M_1\otimes\mat M_1;\label{eq_moment2}\\
\mat M_3 &= \mathbb E[\vct x_1\otimes\vct x_2\otimes\vct x_3]
- \frac{\alpha_0}{\alpha_0+2}(\mathbb E[\vct x_1\otimes\vct x_2\otimes\mat M_1]
+ \mathbb E[\vct x_1\otimes\mat M_1\otimes\vct x_2] + \mathbb E[\mat M_1\otimes\vct x_1\otimes\vct x_2])\nonumber\\
&+ \frac{2\alpha_0^2}{(\alpha_0+1)(\alpha_0+2)}\mat M_1\otimes\mat M_1\otimes\mat M_1.
\label{eq_moment3}
\end{align}
Here $\alpha_0=\sum_k{\alpha_k}$ is assumed to be a known quantity.
Using elementary algebra it can be shown that
\begin{align}
\mat M_2 &= \frac{1}{\alpha_0(\alpha_0+1)}\sum_{i=1}^k{\alpha_i\vct\mu_i\vct\mu_i^\top};\\
\mat M_3 &= \frac{2}{\alpha_0(\alpha_0+1)(\alpha_0+2)}\sum_{i=1}^k{\alpha_i\vct\mu_i\otimes\vct\mu_i\otimes\vct\mu_i}.
\end{align}

To extract topic vectors $\{\vct\mu_i\}_{i=1}^k$ from $\mat M_2$ and $\mat M_3$, a \emph{simultaneous diagonalization} procedure is carried out.
More specifically, the algorithm first finds a whitening matrix $\mat W\in\mathbb R^{V\times K}$ with orthonormal columns such that
$\mat W^\top\mat M_2\mat W = \mat I_{K\times K}$.
In practice, this step can be completed by performing a truncated SVD on $\mat M_2$, $\mat M_2=\mat U_K\mat\Sigma_K\mat V_K$,
and set $\mat W_{ik} = \mat U_{ik}/\sqrt{\mat\Sigma_{kk}}$.
Afterwards, tensor CP decomposition is performed on the whitened third order moment $\mat M_3(\mat W,\mat W,\mat W)$
\footnote{For a tensor $\mat T\in\mathbb R^{V\times V\times V}$ and a matrix $\mat W\in\mathbb R^{V\times k}$,
the product $\mat Q=\mat T(\mat W,\mat W,\mat W)\in\mathbb R^{k\times k\times k}$ is defined as
$\mat Q_{i_1,i_2,i_3}=\sum_{j_1,j_2,j_3=1}^V{\mat T_{j_1,j_2,j_3}\mat W_{j_1,i_1}\mat W_{j_2,i_2}\mat W_{j_3,i_3}}$.}
to obtain a set of eigenvectors $\{\vct v_k\}_{k=1}^K$.
The topic vectors $\{\vct\mu_k\}_{k=1}^K$ can be subsequently obtained by multiplying $\{\vct v_k\}_{k=1}^K$ with the pseudoinverse of $\mat W$.
Note that Eq. (\ref{eq_moment1},\ref{eq_moment2},\ref{eq_moment3}) are defined in exact word moments.
In practice we use empirical moments (e.g., word frequency vector and co-occurrence matrix) to approximate these exact moments.

\subsection{Fast spectral LDA}

To further accelerate the spectral method mentioned in the previous section, it helps to first identify computational bottlenecks of spectral LDA.
In general, the computation of $\widehat{\mat M}_1,\widehat{\mat M}_2$ and the whitening step are not the computational bottleneck
when $V$ is not too large and each document is not too long.
The bottleneck comes from the computation of (the sketch of) $\widehat{\mat M}_3(\mat W,\mat W,\mat W)$ and its tensor decomposition.
By Eq. (\ref{eq_moment3}), the computation of $\widehat{\mat M}_3(\mat W,\mat W,\mat W)$ reduces to computing
$\widehat{\mat M}_1^{\otimes 3}(\mat W,\mat W,\mat W)$, $\hat{\mathbb E}[\vct x_1\otimes\vct x_2\otimes\widehat{\mat M}_1](\mat W,\mat W,\mat W)$,
\footnote{and also $\hat{\mathbb E}[\vct x_1\otimes\widehat{\mat M}_1\otimes\vct x_2](\mat W,\mat W,\mat W)$,
 $\hat{\mathbb E}[\widehat{\mat M}_1\otimes\vct x_1\otimes\vct x_2](\mat W,\mat W,\mat W)$ by symmetry.}
and $\hat{\mathbb E}[\vct x_1\otimes\vct x_2\otimes\vct x_3](\mat W,\mat W,\mat W)$.
The first term $\widehat{\mat M}_1^{\otimes 3}(\mat W,\mat W,\mat W)$ poses no particular challenge as it can be written as
$(\mat W^\top\widehat{\mat M}_1)^{\otimes 3}$.
Its sketch can then be efficiently obtained by applying techniques in Section \ref{sec:tensor_sketch_build}.
In the remainder of this section we focus on efficient computation of the sketch of the other two terms mentioned above.

We first show how to efficiently sketching $\hat{\mathbb E}[\vct x_1\otimes\vct x_2\otimes\vct x_3](\mat W,\mat W,\mat W)$
given the whitening matrix $\mat W$ and $D$ training documents.
Let $\mat T\hat{\mathbb E}[\vct x_1\otimes\vct x_2\otimes\vct x_3](\mat W,\mat W,\mat W)$ denote the whitened $k\times k\times k$ tensor to be sketched
and write $\mat T=\sum_{d=1}^D{\mat T_d}$, where $\mat T_d$ is the contribution of the $d$th training document to $\mat T$.
By definition, $\mat T_d$ can be expressed as $\mat T_d=\mat N_d(\mat W,\mat W,\mat W)$,
where $\mat W$ is the $V\times k$ whitening matrix and $\mat N_d$ is the $V\times V\times V$ empirical moment tensor computed on the $d$th document.
More specifically, for $i,j,k\in\{1,\cdots,V\}$ we have
$$
\mat N_{d,ijk} = \frac{1}{m_d(m_d-1)(m_d-2)}\left\{\begin{array}{ll}
n_{di}(n_{dj}-1)(n_{dk}-2),& i=j=k;\\
n_{di}(n_{di}-1)n_{dk},& i=j, j\neq k;\\
n_{di}n_{dj}(n_{dj}-1)& j=k, i\neq j;\\
n_{di}(n_{di}-1)n_{dj},& i=k, i\neq j;\\
n_{di}n_{dj}n_{dk}, & \text{otherwise}.
\end{array}\right.
$$
Here $m_d$ is the length (i.e., number of words) of document $d$
and $\vct n_d\in\mathbb R^{V}$ is the corresponding word count vector.
Previous straightforward implementation require at least $O(k^3+m_dk^2)$ operations per document to build the tensor $\mat T$
and $O(k^4LT)$ to decompose it \citeapp{spectral-slda,scalable-spectral-lda},
which is prohibitively slow for real-world applications.
In section \ref{sec:tensor_decomposition} we discussed how to decompose a tensor efficiently once we have its sketch.
We now show how to build the sketch of $\mat T$ efficiently from document word counts $\{\vct n_d\}_{d=1}^D$.

By definition, $\mat T_d$ can be decomposed as
\begin{equation}
\mat T_d = \vct p^{\otimes 3}
- \sum_{i=1}^V{n_i(\vct w_i\otimes\vct w_i\otimes\vct p}
{+\vct w_i\otimes\vct p\otimes\vct w_i}
{+\vct p\otimes\vct w_i\otimes\vct w_i)}
+ \sum_{i=1}^V{2n_i\vct w_i^{\otimes 3}},
\label{eq_td}
\end{equation}
where $\vct p=\mat W\vct n$ and $\vct w_i\in\mathbb R^k$ is the $i$th row of the whitening matrix $\mat W$.
A direct implementation is to sketch each of the low-rank components in Eq. (\ref{eq_td}) and compute their sum.
Since there are $O(m_d)$ tensors, building the sketch of $\mat T_d$ requires $O(m_d)$ FFTs, which is unsatisfactory.
However,
note that $\{\vct w_i\}_{i=1}^V$ are fixed and shared across documents.
So when scanning the documents we maintain the sum of $n_i$ and $n_i\vct p$
and add the incremental after all documents are scanned.
In this way, we only need $O(1)$ FFT per document
%\footnote{If symmetric tensor sketch is used then exactly one FFT is required for each document.}
with an additional $O(V)$ FFTs.
Since the total number of documents $D$ is usually much larger than $V$,
this provides significant speed-ups over the naive method that sketches each term in Eq. (\ref{eq_td}) independently.
As a result, the sketch of $\mat T$ can be computed in $O(k(\sum_d{m_d}) + (D+V)b\log b)$ operations,
which is much more efficient than the $O(k^2(\sum_d{m_d})+Dk^3)$ brute-force computation.

We next turn to the term $\hat{\mathbb E}[\vct x_1\otimes\vct x_2\otimes\widehat{\mat M}_1](\mat W,\mat W,\mat W)$.
Fix a document $d$ and let $\vct p=\mat W\vct n_d$.
Define $\vct q = \mat W\widehat{\mat M}_1$.
By definition, the whitened empirical moment can be decomposed as
\begin{equation}
\hat{\mathbb E}[\vct x_1\otimes\vct x_2\otimes\widehat{\mat M}_1](\mat W,\mat W,\mat W) = \sum_{i=1}^V{n_i\vct p\otimes\vct p\otimes\vct q},
\label{eq_applda_exxm}
\end{equation}
Note that Eq. (\ref{eq_applda_exxm}) is very similar to Eq. (\ref{eq_td}).
Consequently, we can apply the same trick (i.e., adding $\vct p$ and $n_i\vct p$ up before doing sketching or FFT) to compute Eq. (\ref{eq_applda_exxm}) efficiently.

\section{Proofs}\label{appsec:proof}

\subsection{Proofs of some technical propositions}\label{appsec:proof_technical}

\begin{proof}[Proof of Proposition \ref{prop_hash_independent}]
We prove the proposition for the case $q=2$ (i.e., $\tilde H$ is 2-wise independent).
This suffices for our purpose in this paper and generalization to $q>2$ cases is straightforward.
For notational simplicity we omit all modulo operators.
Consider two $p$-tuples $\vct l=(l_1,\cdots, l_p)$ and $\vct l'=(l_1',\cdots,l_p')$ such that $\vct l\neq\vct l'$.
Since $\tilde H$ is permutation invariant, we assume without loss of generality that for some $s<p$ and $1\leq i\leq s$ we have $l_i=l_i'$.
Fix $t,t'\in[b]$.
We then have
\begin{multline}
\Pr[\tilde H(\vct l)=t\wedge\tilde H(\vct l')=t'] =
\sum_a\sum_{h(l_1)+\cdots+ h(l_s)=a}{\Pr[ h(l_1)+\cdots+ h(l_s)=a]}\\
\cdot \sum_{\substack{r_{s+1}+\cdots+r_p=t-a\\r_{s+1}'+\cdots+r_p'=t'-a}}{\Pr[ h(l_{s+1})=r_1\wedge\cdots\wedge h(l_p)=r_p\wedge
 h(l_{s+1}')=r_1'\wedge\cdots\wedge h(l_p')=r_p']}.
\label{eq_hash_proof}
\end{multline}
Since $ h$ is $2p$-wise independent, we have
$$\Pr[h(l_1)+\cdots+h(l_s)=a] = \sum_{r_1+\cdots+r_s=a}{\Pr[ h(l_1)=r_1\wedge\cdots h(l_s)=r_s]} = b^{s-1}\cdot\frac{1}{b^s} = \frac{1}{b};$$
\begin{multline*}
\sum_{\substack{r_{s+1}+\cdots+r_p=t-a\\r_{s+1}'+\cdots+r_p'=t-a}}{\Pr[ h(l_{s+1})=r_1\wedge\cdots\wedge h(l_p)=r_p\wedge
 h(l_{s+1}')=r_1'\wedge\cdots\wedge h(l_p')=r_p']} \\
 = b^{2(p-s-1)}\cdot\frac{1}{b^{2(p-s)}} = \frac{1}{b^2}.
\end{multline*}
Summing everything up we get $\Pr[\tilde H(\vct l)=t\wedge\tilde H(\vct l')=t'] = 1/b^2$, which is to be demonstrated.
\end{proof}

\begin{proof}[Proof of Proposition \ref{prop_tiuu}]
Since both FFT and inverse FFT preserve inner products, we have
\begin{align}
%\left[\mat T(\mat I,\vct u, \vct u)\right]_i
\langle\vct s_{\mat T}, \vct s_{1,\vct u}*\vct s_{2,\vct u}*\vct s_{3,\vct e_i}\rangle
&= \langle\mathcal F(\vct s_{\mat T}), \mathcal F(\vct s_{1,\vct u})\circ\mathcal F(\vct s_{2,\vct u})\circ\mathcal F(\vct s_{3,\vct e_i})\rangle\nonumber\\
&= \langle\mathcal F(\vct s_{\mat T})\circ\overline{\mathcal F(\vct s_{1,\vct u})}\circ\overline{\mathcal F(\vct s_{2,\vct u})}, \mathcal F(\vct s_{3,\vct e_i})\rangle\nonumber\\
&= \langle\mathcal F^{-1}(\mathcal F(\vct s_{\mat T})\circ\overline{\mathcal F(\vct s_{1,\vct u})}\circ\overline{\mathcal F(\vct s_{2,\vct u})}), \vct s_{3,\vct e_i}\rangle.\nonumber
\label{eq_tiuv_fast}
\end{align}
\end{proof}

\subsection{Analysis of tensor sketch approximation error}\label{appsec:tensor_sketch_bound}

Proofs of Theorem \ref{thm_tensor_error} is based on the following two key lemmas,
which states that $\langle\tilde{\vct s}_{\mat A},\tilde{\vct s}_{\widetilde{\mat B}}\rangle$
is a consistent estimator of the true inner product $\langle\mat A,\mat B\rangle$;
furthermore, the variance of the estimator decays linearly with the hash length $b$.
The lemmas are interesting in their own right, providing useful tools for proving approximation accuracy in a wide range of applications
when colliding hash and symmetric sketches are used.
\begin{lem}
Suppose $\mat A,\mat B\in\bigotimes^p\mathbb R^n$ are two symmetric real tensors and let $\tilde{\vct s}_{\mat A},\tilde{\vct s}_{\widetilde{\mat B}}\in\mathbb C^b$
be the symmetric tensor sketches of $\mat A$ and $\widetilde{\mat B}$. That is,
\begin{eqnarray}
\tilde s_{\mat A}(t) &=& \sum_{\tilde H(i_1,\cdots,i_p)=t}{\sigma_{i_1}\cdots\sigma_{i_p}\mat A_{i_1,\cdots,i_p}};\label{eq_e_ver1}\\
\tilde s_{\widetilde{\mat B}}(t) &=& \sum_{\substack{\tilde H(i_1,\cdots,i_p)=t\\i_1\leq \cdots\leq i_p}}{\sigma_{i_1}\cdots\sigma_{i_p}\mat B_{i_1,\cdots,i_p}}.\label{eq_var_ver1}
\end{eqnarray}
Assume $\tilde H(i_1,\cdots,i_p) = (h(i_1)+\cdots+h(i_p))\mod b$ are drawn from a 2-wise independent hash family.
Then the following holds:
\begin{eqnarray}
\mathbb E_{h,\sigma}\left[\langle\tilde{\vct s}_{\mat A},\tilde{\vct s}_{\widetilde{\mat B}}\rangle\right] &=& \langle\mat A,\mat B\rangle,\label{eq_error_e}\\
\mathbb V_{h,\sigma}\left[\langle\tilde{\vct s}_{\mat A},\tilde{\vct s}_{\widetilde{\mat B}}\rangle\right] &\leq& \frac{4^p\|\mat A\|_F^2\|\mat B\|_F^2}{b}.\label{eq_error_var}
\end{eqnarray}
\label{lem_main_1}
\end{lem}

\begin{lem}
Following notations and assumptions in Lemma \ref{lem_main_1}.
Let $\{\mat A_i\}_{i=1}^m$ and $\{\mat B_i\}_{i=1}^m$ be symmetric real $n\times n\times n$ tensors
and fix real vector $\vct w\in\mathbb R^m$.
Then we have
\begin{eqnarray}
\mathbb E\left[\sum_{i,j}w_iw_j\langle\tilde{\vct s}_{\mat A_i},\tilde{\vct s}_{\widetilde{\mat B}_j}\rangle\right] &=& \sum_{i,j}{w_iw_j\langle\mat A_i,\mat B_j\rangle};\label{eq_e_ver2}\\
\mathbb V\left[\sum_{i,j}w_iw_j\langle\tilde{\vct s}_{\mat A_i},\tilde{\vct s}_{\widetilde{\mat B}_j}\rangle\right] &\leq& \frac{4^p\|\vct w\|^4(\max_i{\|\mat A_i\|_F^2})(\max_i{\|\mat B_i\|_F^2})}{b}.
\label{eq_var_ver2}
\end{eqnarray}
\label{lem_main_2}
\end{lem}

\begin{proof}[Proof of Lemma \ref{lem_main_1}]
We first define some notations. Let $\vct l=(l_1,\cdots,l_p)\in [d]^p$ be a $p$-tuple denoting a multi-index.
Define $\mat A_{\vct l} := \mat A_{l_1,\cdots,l_p}$ and $\sigma(\vct l) := \sigma_{l_1}\cdots\sigma_{l_p}$.
For $\vct l,\vct l'\in [n]^p$, define $\delta(\vct l,\vct l') = 1$ if $h(l_1)+\cdots+h(l_p)\equiv h(l_1')+\cdots+h(l_p') (\mod b)$
and $\delta(\vct l,\vct l') = 0$ otherwise.
For a $p$-tuple $\vct l\in[n]^p$, let $\mathcal L(\vct l)\in[n]^p$ denote the $p$-tuple
obtained by re-ordering indices in $\vct l$ in ascending order.
Let $\mathcal M(\vct l)\in\mathbb N^b$ denote the ``expanded version'' of $\vct l$. That is,
$[\mathcal M(\vct l)]_i$ denote the number of occurrences of the index $i$ in $\vct l$.
By definition, $\|\mathcal M(\vct l)\|_1 = p$.
Finally, by definition $\widetilde{\mat B}_{\vct l'} = \mat B_{\vct l'}$ if $\vct l'=\mathcal L(\vct l')$ and $\widetilde{\mat B}_{\vct l'} = 0$ otherwise.

Eq. (\ref{eq_error_e}) is easy to prove. By definition and linearity of expectation we have
\begin{equation}
\mathbb E[\langle\tilde{\vct s}_{\mat A},\tilde{\vct s}_{\widetilde{\mat B}}\rangle] = \sum_{\vct l,\vct l'}{\delta(\vct l,\vct l')\sigma(\vct l)\mat A_{\vct l}\bar\sigma(\vct l')\widetilde{\mat B}_{\vct l'}}.
\end{equation}
Note that $\delta$ and $\sigma$ are independent and
\begin{equation}
\mathbb E_{\sigma}[\sigma(\vct l)\sigma(\vct l')]  = \left\{\begin{array}{ll}
1,& \text{if }\mathcal L(\vct l)=\mathcal L(\vct l');\\
0,& \text{otherwise.}\end{array}\right.
\end{equation}
Also $\delta(\vct l,\vct l') = 1$ with probability 1 whenever $\mathcal L(\vct l) = \mathcal L(\vct l')$.
Note that $\widetilde{\mat B}_{\vct l'} = 0$ whenever $\vct l'\neq\mathcal L(\vct l')$.
Consequently,
\begin{equation}
\mathbb E[\langle\tilde{\vct s}_{\mat A},\tilde{\vct s}_{\widetilde{\mat B}}\rangle] = \sum_{\vct l\in[n]^p}{\mat A_{\vct l}\widetilde{\mat B}_{\mathcal L(\vct l)}} = \langle\mat A,\mat B\rangle.
\end{equation}

For the variance, we have the following expression for $\mathbb E[\langle\tilde{\vct s}_{\mat A},\tilde{\vct s}_{\widetilde{\mat B}}\rangle^2]$:
\begin{eqnarray}
\mathbb E[\langle\tilde{\vct s}_{\mat A},\tilde{\vct s}_{\widetilde{\mat B}}\rangle^2]
%&=& \mathbb E[\langle\tilde{\vct s}_{\mat A},\tilde{\vct s}_{\widetilde{\mat B}}\rangle\overline{\langle\tilde{\vct s}_{\mat A},\tilde{\vct s}_{\widetilde{\mat B}}\rangle}]\\
&=& \sum_{\vct l,\vct l',\vct r,\vct r'}{\mathbb E[\delta(\vct l,\vct l')\delta(\vct r,\vct r')]\cdot\mathbb E[\sigma(\vct l)\bar\sigma(\vct l')\bar\sigma(\vct r)\sigma(\vct r')]\cdot \mat A_{\vct l}\mat A_{\vct r}\widetilde{\mat B}_{\vct l'}\widetilde{\mat B}_{\vct r'}}\\
&=:& \sum_{\vct l,\vct l',\vct r,\vct r'}{\mathbb E[t(\vct l,\vct l',\vct r,\vct r')]}.
\end{eqnarray}

We remark that $\mathbb E[\sigma(\vct l)\bar\sigma(\vct l')\bar\sigma(\vct r)\sigma(\vct r')] = 0$ if $\mathcal M(\vct l)-\mathcal M(\vct l') \neq \mathcal M(\vct r) - \mathcal M(\vct r')$.
In the remainder of the proof we will assume that $\mathcal M(\vct l) -\mathcal M(\vct l')= \mathcal M(\vct r) - \mathcal M(\vct r')$.
This can be further categorized into two cases:

\textbf{Case 1}: $\vct l' = \mathcal L(\vct l)$ and $\vct r' = \mathcal L(\vct r)$.
By definition $\mathbb E[\sigma(\vct l)\bar\sigma(\vct l')\sigma(\vct r)\bar\sigma(\vct r')] = 1$ and $\mathbb E[\delta(\vct l,\vct l')\delta(\vct r,\vct r')] = 1$.
Subsequently $\mathbb E[t(\vct l,\vct l',\vct r,\vct r')] =  \mat A_{\vct l}\mat A_{\vct r}\widetilde{\mat B}_{\vct l'}\widetilde{\mat B}_{\vct r'}$ and hence
\begin{equation}
\sum_{\vct l,\vct r,\vct l'=\mathcal L(\vct l),\vct r'=\mathcal L(\vct r)}{\mathbb E[t(\vct l,\vct l',\vct r,\vct r')]}
= \sum_{\vct l,\vct r}{\mat A_{\vct l}\mat A_{\vct r}\mat B_{\vct l}\mat B_{\vct r}} = \langle\mat A,\mat B\rangle^2.
\label{eq_var_case1}
\end{equation}

\textbf{Case 2}: $\vct l'\neq\mathcal L(\vct l)$ or $\vct r'\neq\mathcal L(\vct r)$.
Since $\mathcal M(\vct l) -\mathcal M(\vct l')= \mathcal M(\vct r) - \mathcal M(\vct r')\neq 0$ we have
$\mathbb E[\delta(\vct l,\vct l')\delta(\vct r,\vct r')] = 1/b$ because $h$ is a 2-wise independent hash function.
In addition, $\mathbb E[|\sigma(\vct l)\bar\sigma(\vct l')\sigma(\vct r)\bar\sigma(\vct r')|] \leq 1$.

To enumerate all $(\vct l,\vct l',\vct r,\vct r')$ tuples that satisfy the colliding condition $\mathcal M(\vct l) -\mathcal M(\vct l')= \mathcal M(\vct r) - \mathcal M(\vct r')\neq 0$,
we fix
\footnote{Note that $\text{sum}(\mathcal M(\vct l))=\text{sum}(\mathcal M(\vct l'))$ and hence $\|\mathcal M(\vct l)-\mathcal M(\vct l')\|_1$ must be even.
Furthermore, the sum of positive entries in $(\mathcal M(\vct l)-\mathcal M(\vct l'))$ equals the sum of negative entries.}
 $\|\mathcal M(\vct l)-\mathcal M(\vct l')\|_1 = 2q$
 and fix $q$ positions each in $\vct l$ and $\vct r$ (for $\vct l'$ and $\vct r'$ the positions of these indices
 are automatically fixed because indices in $\vct l'$ and $\vct r'$ must be in ascending order).
Without loss of generality assume the fixed $q$ positions for both $\vct l$ and $\vct r$ are the first $q$ indices.
The 4-tuple $(\vct l,\vct r,\vct l',\vct r')$ with $\|\mathcal M(\vct l)-\mathcal M(\vct l')\|_1=2q$ can then be enumerated as follows:
\begin{eqnarray}
&&\sum_{\substack{\vct l,\vct r,\vct l',\vct r'\\
\mathcal M(\vct l)-\mathcal M(\vct l')=\mathcal M(\vct r)-\mathcal M(\vct r')\\
\|\mathcal M(\vct l)-\mathcal M(\vct l')\|_1 = 2q}}{t(\vct l,\vct l',\vct r,\vct r')}\nonumber\\
&=& \sum_{\vct i\in[n]^q}{\sum_{\vct j\in[n]^q}{\sum_{\substack{\vct l\in[n]^{p-q}\\\vct r\in[n]^{p-q}}}{t(\vct i\circ\vct l,\mathcal L(\vct j\circ\vct l),\vct i\circ\vct r,\mathcal L(\vct j\circ\vct r))}}}\nonumber\\
&\leq&\frac{1}{b} \sum_{\substack{\vct i,\vct j\in[n]^q\\\vct l,\vct r\in[n]^{p-q}}}{\mat A_{\vct i\circ\vct l}\mat A_{\vct i\circ\vct r}\mat B_{\vct j\circ\vct l}\mat B_{\vct j\circ\vct r}}\nonumber\\
&=& \frac{1}{b}\sum_{\substack{\vct i,\vct j\in[n]^q}}{\langle\mat A(\vct e_{i_1},\cdots,\vct e_{i_q},\mat I,\cdots,\mat I), \mat B(\vct e_{j_1},\cdots,\vct e_{j_q},\mat I,\cdots,\mat I)\rangle^2}\nonumber\\
&\leq& \frac{1}{b}\sum_{\substack{\vct i,\vct j\in[n]^q}}{\|\mat A(\vct e_{i_1},\cdots,\vct e_{i_q},\mat I,\cdots,\mat I)\|_F^2\|\mat B(\vct e_{j_1},\cdots,\vct e_{j_q},\mat I,\cdots,\mat I)\|_F^2}\nonumber\\
&=& \frac{\|\mat A\|_F^2\|\mat B\|_F^2}{b}.
\label{eq_var_case2}
\end{eqnarray}
Here $\circ$ denotes concatenation, that is, $\vct i\circ\vct l = (i_1,\cdots,i_q,l_1,\cdots,l_{p-q})\in[n]^p$.
The fourth equation is Cauchy-Schwartz inequality.
Finally note that there are no more than $4^p$ ways of assigning $q$ positions to $\vct l$ and $\vct l'$ each.
Combining Eq. (\ref{eq_var_case1}) and (\ref{eq_var_case2}) we get
\begin{equation*}
\mathbb V[\langle\tilde{\vct s}_{\mat A},\tilde{\vct s}_{\widetilde{\mat B}}\rangle]
= \mathbb E[\langle\tilde{\vct s}_{\mat A},\tilde{\vct s}_{\widetilde{\mat B}}\rangle^2] - \langle\mat A,\mat B\rangle^2
\leq \frac{4^p\|\mat A\|_F^2\|\mat B\|_F^2}{b},
\end{equation*}
which completes the proof.
\end{proof}

\begin{proof}[Proof of Lemma \ref{lem_main_2}]
Eq. (\ref{eq_e_ver2}) immediately follows Eq. (\ref{eq_e_ver1}) by adding everything together.
For the variance bound we cannot use the same argument because in general the $m^2$ random variables are neither independent nor uncorrelated.
Instead, we compute the variance by definition.
First we compute the expected square term as follows:
\begin{multline}
\mathbb E\left[\left(\sum_{i,j}w_iw_j\langle\tilde{\vct s}_{\mat A_i},\tilde{\vct s}_{\widetilde{\mat B}_j}\rangle\right)^2\right] \\
= \sum_{\substack{i,j,i',j'\\\vct l,\vct l',\vct r,\vct r'}}{w_iw_jw_{i'}w_{j'}\cdot \mathbb E[\delta(\vct l,\vct l')\delta(\vct r,\vct r')]\cdot
\mathbb E[\sigma(\vct l)\bar\sigma(\vct l')\bar\sigma(\vct r)\sigma(\vct r')]\cdot
[\mat A_i]_{\vct l}[\mat A_{i'}]_{\vct r}[\widetilde{\mat B}_j]_{\vct l'}[\widetilde{\mat B}_{j'}]_{\vct r'}}.
\end{multline}
Define $\mat X=\sum_i{w_i\mat A_i}$ and $\mat Y=\sum_i{w_i\mat B_i}$.
The above equation can then be simplified as
\begin{equation}
\mathbb E\left[\left(\sum_{i,j}w_iw_j\langle\tilde{\vct s}_{\mat A_i},\tilde{\vct s}_{\widetilde{\mat B}_j}\rangle\right)^2\right]
= \sum_{\vct l,\vct l',\vct r,\vct r'}{\mathbb E[\delta(\vct l,\vct l')\delta(\vct r,\vct r')]\cdot \mathbb E[\sigma(\vct l)\bar\sigma(\vct l')\bar\sigma(\vct r)\sigma(\vct r')]\cdot
\mat X_{\vct l}\mat X_{\vct r}\widetilde{\mat Y}_{\vct l'}\widetilde{\mat Y}_{\vct r'}}.
\end{equation}
Applying Lemma \ref{lem_main_1} we have
\begin{equation}
\mathbb V\left[\sum_{i,j}w_iw_j\langle\tilde{\vct s}_{\mat A_i},\tilde{\vct s}_{\widetilde{\mat B}_j}\rangle\right]
\leq \frac{4^p\|\mat X\|_F^2\|\mat Y\|_F^2}{b}.
\end{equation}
Finally, note that
\begin{equation}
\|\mat X\|_F^2 = \sum_{i,j}{w_iw_j\langle\mat A_i,\mat A_j\rangle} \leq \sum_{i,j}{w_iw_j\|\mat A_i\|_F\|\mat A_j\|_F} \leq \|\vct w\|^2\max_i{\|\mat A_i\|_F^2}.
\end{equation}

\end{proof}

With Lemma \ref{lem_main_1} and \ref{lem_main_2}, we can easily prove Theorem \ref{thm_tensor_error}.
\begin{proof}[Proof of Theorem \ref{thm_tensor_error}]
First we prove the $\varepsilon_1(\vct u)$ bound.
Let $\mat A=\mat T$ and $\mat B = \vct u^{\otimes 3}$. Note that $\|\mat A\|_F = \|\mat T\|_F$ and $\|\mat B\|_F = \|\vct u\|^2 = 1$.
Note that $[\mat T(\mat I,\vct u,\vct u)]_i = \mat T(\vct e_i,\vct u,\vct u)$.
Next we consider $\vct\varepsilon_2(\vct u)$ and let $\mat A=\mat T$, $\mat B = \vct e_i\otimes\vct u\otimes\vct u$.
Again we have $\|\mat A\|_F = \|\mat T\|_F$ and $\|\mat B\|_F = 1$.
A union bound over all $i=1,\cdots, n$ yields the result.
For the inequality involving $\vct w$ we apply Lemma \ref{lem_main_2}.
\end{proof}

\subsection{Analysis of fast robust tensor power method}\label{appsec:tensor_power_method}

In this section, we prove Theorem \ref{thm_refine_tensor_power}, a more refined version of Theorem \ref{thm_tensor_power} in Section \ref{subsec:analysis_tensor_power_method}.
We structure the section by first demonstrating the convergence behavior of noisy tensor power method,
and then show how error accumulates with deflation. Finally, the overall bound is derived by combining these two parts.

\subsubsection{Recovering the principal eigenvector}
Define the angle between two vectors $\vct v$ and $\vct u$ to be $\theta \rbr{\vct v, \vct u}.$
First, in Lemma \ref{lem_first_phase} we show that
if the initialization vector $\vct u_0$ is randomly chosen from the unit sphere,
then the angle $\theta$ between the iteratively updated vector $\vct u_t$ and the largest eigenvector of tensor $\mat T,$ $\vct v_1,$ will decrease
to a point that $\tan \theta \rbr{\vct v_1, \vct u_t} < 1$.
Afterwards, in Lemma \ref{lem_second_phase}, we use a similar approach as in \citeapp{noisy-tensor-power-method} to prove that the error between the final estimation
and the ground truth is bounded.

Suppose $\mat T$ is the exact low-rank ground truth tensor and
Each noisy tensor update can then be written as
\begin{equation}
\tilde{\vct u}_{t+1} = \mat T(\mat I,\vct u_t,\vct u_t) + \tilde{\vct\varepsilon}(\vct u_t),
\end{equation}
where $\tilde{\vct\varepsilon}(\vct u_t) = \mat E(\mat I,\vct u_t,\vct u_t) + \vct\varepsilon_{2,T}(\vct u_t)$
is the noise coming from statistical and tensor sketch approximation error.

{
Before presenting key lemmas, we first define \emph{$\gamma$-separation}, a concept introduced in \citeapp{tensor-power-method}.
\begin{defn}[$\gamma$-separation, \citeapp{tensor-power-method}]
Fix $i^*\in[k]$, $\vct u\in\mathbb R^n$ and $\gamma>0$.
$u$ is \emph{$\gamma$-separated} with respect to $\vct v_{i^*}$ if the following holds:
\begin{equation}
\lambda_{i^*}\langle\vct u, \vct v_{i^*}\rangle - \max_{i\in[k]\backslash\{i^*\}}{\lambda_i\langle\vct u,\vct v_i\rangle}
\geq \gamma\lambda_{i^*}\langle\vct u, \vct v_{i^*}\rangle.
\end{equation}
\end{defn}

Lemma \ref{lem_first_phase} analyzes the first phase of the noisy tensor power algorithm.
It shows that if the initialization vector $\vct u_0$ is $\gamma$-separated with respect to $\vct v_1$ and the magnitude of noise $\tilde{\vct\varepsilon}(\vct u_t)$ is small
at each iteration $t$, then after a short number of iterations we will have inner product between $\vct u_t$ and $\vct v_1$ at least a constant.
}

\begin{lem}
\label{lem_first_phase}
Let $\cbr{\vct v_1, \vct v_2, \cdots, \vct v_k}$ and $\cbr{ \lambda_1, \lambda_2, \cdots, \lambda_k}$ be eigenvectors and eigenvalues of tensor $\mat T \in \mathbb{R}^{n \times n \times n},$
where $\lambda_1 \abr{\inner{\vct v_1}{\vct u_0}} = \max \limits_{i \in [k]} \lambda_i \abr{\inner{\vct v_i}{\vct u_0}}.$
Denote $\mat V=(\vct v_1,\cdots,\vct v_k)\in\mathbb R^{n\times k}$ as the matrix for eigenvectors.
Suppose that for every iteration $t$ the noise satisfies
\begin{align}
  \big|\langle\vct v_i,\tilde{\vct\varepsilon}(\vct u_t)\rangle\big| \leq \epsilon_1\,\, \, \forall \, i \in [n] \,\,\, \text{and}\,\,\,
  \nbr{\mat V^\top\tilde{\vct\varepsilon}(\vct u_t)} \leq \epsilon_2;
\end{align}
suppose also the initialization $\vct u_0$ is $\gamma$-separated with respect to $\vct v_1$ for some $\gamma \in (0.5,1)$. If $\tan \theta \rbr{\vct v_1, \vct u_0} > 1,$ and
\begin{align}
\epsilon_1 \leq \min \rbr{\frac{1}{4\frac{\max_{i \in [k]} \lambda_i}{\lambda_1} +2}, \frac{1- (1+\alpha)/2}{2} } \lambda_1 \inner{\vct v_1}{\vct u_0}^2 \,\,\, \text{and}\,\,\, \epsilon_2 \leq  \frac{1- (1+\alpha) /2}{2 \sqrt{2}(1+\alpha)}  \lambda_1 \abr{\inner{\vct v_1}{\vct u_0}}
\end{align}
for some $\alpha >0$,
then for a small constant $\rho >0$, there exists a $T > \log_{1+\alpha} \rbr{1+\rho}{\tan \theta \rbr{\vct v_1, \vct u_0}}$ such that after $T$ iteration, we have $\tan \theta \rbr{\vct v_1, \vct u_T} < \frac{1}{1+\rho},$
\end{lem}
\begin{proof}
   Let $\tilde{\vct u}_{t+1} = {\mat T} \rbr{\mat I, \vct u_t, \vct u_t} + \tilde{\vct\varepsilon}(\vct u_t)$ and $\vct u_{t+1} = \tilde{\vct u}_{t+1} / \nbr{\tilde{\vct u}_{t+1}}.$ For $\alpha \in (0,1), $ we try to prove that there exists a $T$ such that for $t > T$
   \begin{align}
   \frac{1}{\tan \theta \rbr{\vct v_1, \vct u_{t+1}}} = \frac{\abr{\inner{\vct v_1}{\vct u_{t+1}}}}{ \rbr{1- \inner{\vct v_1}{\vct u_{t+1}}^2 }^{1/2}} = \frac{\abr{\inner{\vct v_1}{ \tilde{\vct u}_{t+1}}}}{ \rbr{\sum \limits_{i=2}^n \inner{\vct v_i}{\tilde{\vct u}_{t+1}}^2 }^{1/2}} \geq 1.
   \end{align}
First we examine the numerator. Using the assumption $\big|\langle\vct v_i,\tilde{\vct\varepsilon}(\vct u_t)\rangle\big|\leq  \epsilon_1$
and the fact that $\inner{\vct v_i}{\tilde{\vct u}_{t+1}} = \lambda_i \inner{\vct v_i}{\vct u_t}^2 + \langle\vct v_i,\tilde{\vct\varepsilon}(\vct u_t)\rangle,$ we have
\begin{align}
\label{eq_projection_cos}
    \abr{\inner{\vct v_i}{\tilde{\vct u}_{t+1}}} \geq \lambda_i \inner{\vct v_i}{\vct u_t}^2 - \epsilon_1 \geq \abr{\inner{\vct v_i}{\vct u_t}} \rbr{\lambda_i \abr{\inner{\vct v_i}{\vct u_t}} - \epsilon_1/\abr{\inner{\vct v_i}{\vct u_t}}}.
\end{align}
For the denominator, by H\"older's inequality we have
\begin{align}
\label{eq_projection_sin}
 \rbr{\sum \limits_{i=2}^n \inner{\vct v_i}{\tilde{\vct u}_{t+1}}^2 }^{1/2}
  &= \rbr{\sum \limits_{i = 2}^n  \rbr{ \lambda_i \inner{\vct v_i}{\vct u_t}^2 + \langle\vct v_i,\tilde{\vct\varepsilon}(\vct u_t)\rangle }^{1/2} } \\
 &\leq \rbr{\sum \limits_{i = 2}^n \lambda_i^2 \inner{\vct v_i}{\vct u_t}^4}^{1/2} + \rbr{\sum \limits_{i = 2}^n \langle\vct v_i,\tilde{\vct\varepsilon}(\vct u_t)\rangle  ^2}^{1/2}\\
 &\leq \max \limits_{i \neq 1} \lambda_i \abr{\inner{\vct v_i}{\vct u_t}} \rbr{\sum \limits_{i=2}^n \inner{\vct v_i}{\vct u_t}^2}^{1/2} +\epsilon_2\\
 & \leq \rbr{1 - \inner{\vct v_1}{\vct u_t}^2}^{1/2} \rbr{\max \limits_{i \neq 1} \lambda_i \abr{\inner{\vct v_i}{\vct u_t}} + \epsilon_2/\rbr{1 - \inner{\vct v_1}{\vct u_t}^2}^{1/2} }
\end{align}
Equation \eqref{eq_projection_cos} and \eqref{eq_projection_sin} yield
\begin{align}
  \frac{1}{\tan \theta \rbr{\vct v_1, \vct u_{t+1}}}
  &\geq \frac{\abr{\inner{\vct v_1}{\vct u_t}}}{\rbr{1 - \inner{\vct v_1}{\vct u_t}^2}^{1/2}} \frac{ \lambda_1 \abr{\inner{\vct v_1}{\vct u_t}} - \epsilon_1/ \abr{\inner{\vct v_1}{\vct u_t}}}{ \max \limits_{i \neq 1} \lambda_i \abr{\inner{\vct v_i}{\vct u_t}} + \epsilon_2/\rbr{1 - \inner{\vct v_1}{\vct u_t}^2}^{1/2}}\\
  & = \frac{1}{\tan \theta \rbr{\vct v_1, \vct u_t}} \frac{ \lambda_1 \abr{\inner{\vct v_1}{\vct u_t}} - \epsilon_1/ \abr{\inner{\vct v_1}{\vct u_t}}}{ \max \limits_{i \neq 1} \lambda_i \abr{\inner{\vct v_i}{\vct u_t}} + \epsilon_2/\rbr{1 - \inner{\vct v_1}{\vct u_t}^2}^{1/2}}
\end{align}
To prove that the second term is larger than $1 + \alpha,$ we first show that when $t = 0,$ the inequality holds. Since the initialization vector is a $\gamma-$separated vector, we have
\begin{align}
   \lambda_1 \abr{\inner{\vct v_1}{\vct u_0}} - \max \limits_{i \in [k]} \lambda_i \abr{\inner{\vct v_i}{\vct u_0}} &\geq \gamma \lambda_1 \abr{\inner{\vct v_1}{\vct u_0} },\\
   \label{eq_gamma_separation}
   \max \limits_{i \in [k]} \lambda_i \abr{\inner{\vct v_i}{\vct u_0}} &\leq (1-\gamma)  \lambda_1 \abr{\inner{\vct v_1}{\vct u_0}} \leq 0.5 \lambda_1 \abr{\inner{\vct v_1}{\vct u_0}},
\end{align}
the last inequality holds since $\gamma > 0.5.$
Note that we assume $\tan \theta \rbr{\vct v_1, \vct {u_0}} > 1$ and hence $\inner{\vct v_1}{\vct u_0}^2 < 0.5$.
Therefore,
\begin{align}
\epsilon_2 \leq \frac{1- (1+\alpha) /2}{2 \sqrt{2}(1+\alpha)} \lambda_1 \abr{\inner{\vct v_1}{\vct u_0}} \leq \frac{  \rbr{1 - \inner{\vct v_1}{\vct u_0}^2}^{1/2} \rbr{1- (1+\alpha)/2}}{2 (1+\alpha)} \lambda_1 \abr{\inner{\vct v_1}{\vct u_0}}.
\end{align}

\noindent Thus, for $t = 0,$ using the condition for $\epsilon_1$ and $\epsilon_2$ we have
\begin{align}
\label{eq_1_plus_alpha}
\frac{ \lambda_1 \abr{\inner{\vct v_i}{\vct u_0}} - \epsilon_1/ \abr{\inner{\vct v_i}{\vct u_0}}}{ \max \limits_{i \neq 1} \lambda_i \abr{\inner{\vct v_i}{\vct u_0}} + \epsilon_2/\rbr{1 - \inner{\vct v_1}{\vct u_0}^2}^{1/2}}
\geq \frac{ \lambda_1 \abr{\inner{\vct v_i}{\vct u_0}} - \epsilon_1/ \abr{\inner{\vct v_i}{\vct u_0}}}{ 0.5 \lambda_1 \abr{\inner{\vct v_1}{\vct u_0}} + \epsilon_2/\rbr{1 - \inner{\vct v_1}{\vct u_0}^2}^{1/2}} \geq 1+\alpha.
\end{align}
The result yields $1/ \tan \theta \rbr{\vct v_1, \vct u_1} > (1+\alpha)/ \tan \theta \rbr{\vct v_1, \vct u_0}.$ This also indicates that $\abr{\inner{\vct v_1}{\vct u_1}} > \abr{\inner{\vct v_1}{\vct u_0}},$ which implies that
\begin{align}
\label{eq_2_condition}
&\epsilon_1 \leq \min \rbr{\frac{1}{4 \frac{\max_{i \in [k]} \lambda_i}{\lambda_1} + 2}, \frac{1- (1+\alpha)/2}{2} } \lambda_1 \inner{\vct v_1}{\vct u_t}^2\,\,\, \text{and} \,\,\, \epsilon_2 \leq  \frac{1- (1+\alpha)/2}{2 \sqrt{2}(1+\alpha)}  \lambda_1 \abr{\inner{\vct v_1}{\vct u_t}}
\end{align}
also holds for $t = 1.$
Next we need to make sure that for $t \geq 0$
\begin{align}
&\max \limits_{i \neq 1} \lambda_i \abr{\inner{\vct v_i}{\vct u_t}}  \leq 0.5 \lambda_1 \abr{\inner{\vct v_1}{\vct u_t}}.
 \end{align}
In other words, we need to show that $\frac{\lambda_1 \abr{\inner{\vct v_1 }{\vct u_t} }}{ \max \limits_{i \neq 1} \lambda_i \abr{\inner{\vct v_i }{\vct u_t} }} \geq 2.$
From Equation \eqref{eq_gamma_separation}, for $t = 0,$ $\frac{\lambda_1 \abr{\inner{\vct v_1 }{\vct u_t} }}{ \max \limits_{i \neq 1} \lambda_i \abr{\inner{\vct v_i }{\vct u_t} }} \geq \frac{1}{1-\gamma} \geq 2.$ For every $i \in [k],$
\begin{align}
   \abr{\inner{\vct v_i}{\tilde{\vct u}_{t+1}}} \leq \lambda_i \abr{\inner{\vct v_i}{\vct u_t}}^2 + \epsilon_1 \leq \abr{\inner{\vct v_i}{\vct u_t}}\rbr{\lambda_i \abr{\inner{\vct v_i}{\vct u_t}} + \epsilon_1/\abr{\inner{\vct v_i}{\vct u_t}}}.
\end{align}
With equation \eqref{eq_projection_cos}, we have
\begin{align}
\frac{\lambda_1 \abr{\inner{\vct v_1}{\vct u_{t+1}} }}{ \lambda_i \abr{\inner{\vct v_i}{\vct u_{t+1}} }} = \frac{\lambda_1 \abr{\inner{\vct v_1}{\tilde{\vct u}_{t+1}}} }{ \lambda_i \abr{\inner{\vct v_i}{\tilde{\vct u}_{t+1}}}}
&\geq \frac{\lambda_1 \abr{\inner{\vct v_1}{\vct u_t}} \rbr{\lambda_1 \abr{\inner{\vct v_1}{\vct u_t}} - \frac{\epsilon_1}{\abr{\inner{\vct v_1}{\vct u_t} }} } }{ \lambda_i \abr{\inner{\vct v_i}{\vct u_t}} \rbr{\lambda_i \abr{\inner{\vct v_i}{\vct u_t}} - \frac{\epsilon_1}{\abr{\inner{\vct v_i}{\vct u_t} }} }    } \\
&= \rbr{\frac{\lambda_1 \abr{\inner{\vct v_1}{\vct u_{t}} }}{ \lambda_i \abr{\inner{\vct v_i}{\vct u_{t}} }}}^2 \frac{1- \frac{\epsilon_1}{\lambda_1 \inner{\vct v_1}{\vct u_t}^2} }{ 1+ \frac{\lambda_i}{\lambda_1} \frac{\epsilon_1}{\lambda_1 \inner{\vct v_1}{\vct u_t}^2} \rbr{\frac{\lambda_1 \abr{\inner{\vct v_1}{\vct u_{t}} }}{ \lambda_i \abr{\inner{\vct v_i}{\vct u_{t}} }}}^2   }\\
& \geq \rbr{\frac{\lambda_1 \abr{\inner{\vct v_1}{\vct u_{t}} }}{ \lambda_i \abr{\inner{\vct v_i}{\vct u_{t}} }}}^2 \frac{1- \frac{\epsilon_1}{\lambda_1 \inner{\vct v_1}{\vct u_t}^2} }{ 1+ \frac{\max \limits_{i \in [k]} \lambda_i}{\lambda_1} \frac{\epsilon_1}{\lambda_1 \inner{\vct v_1}{\vct u_t}^2} \rbr{\frac{\lambda_1 \abr{\inner{\vct v_1}{\vct u_{t}} }}{ \lambda_i \abr{\inner{\vct v_i}{\vct u_{t}} }}}^2   }\\
& =  \frac{1- \frac{\epsilon_1}{\lambda_1 \inner{\vct v_1}{\vct u_t}^2} }{ \frac{1}{\rbr{\frac{\lambda_1 \abr{\inner{\vct v_1}{\vct u_{t}} }}{ \lambda_i \abr{\inner{\vct v_i}{\vct u_{t}} }}}^2 }+ \frac{\max_{i \in [k]} \lambda_i}{\lambda_1} \frac{\epsilon_1}{\lambda_1 \inner{\vct v_1}{\vct u_t}^2}   }.
\end{align}
Let $\kappa = \frac{\max_{i \in [k]} \lambda_i}{\lambda_1}$. For $t = 0,$ with conditions on $\epsilon_1$ the following holds:
\begin{align}
\frac{\lambda_1 \abr{\inner{\vct v_1}{\vct u_{1}} }}{ \lambda_i \abr{\inner{\vct v_i}{\vct u_{1}} }}
 &\geq \frac{1- \frac{\epsilon_1}{\lambda_1 \inner{\vct v_1}{\vct u_0}^2} }{ \frac{1}{\rbr{\frac{\lambda_1 \abr{\inner{\vct v_1}{\vct u_{0}} }}{ \lambda_i \abr{\inner{\vct v_i}{\vct u_{0}} }}}^2 }+ \frac{\max_{i \in [k]} \lambda_i}{\lambda_1} \frac{\epsilon_1}{\lambda_1 \inner{\vct v_1}{\vct u_0}^2}   }. \\
 & \geq \frac{1 - \frac{1}{4\kappa + 2}}{ \frac{1}{4} + \frac{\kappa}{4\kappa + 2}} =  2
\end{align}
With the two conditions stated in Equation \eqref{eq_2_condition}, following the same step in \eqref{eq_1_plus_alpha}, we have $\frac{1}{\tan \theta \rbr{\vct v_1,u_2 }} \geq (1+\alpha) \frac{1}{\tan \theta \rbr{\vct v_1,u_1 }}.$
 By induction, $\frac{1}{\tan \theta \rbr{\vct v_1,u_{t+1} }} \geq (1+\alpha) \frac{1}{\tan \theta \rbr{\vct v_1,t }}.$
 for $t \geq 0$.
Subsequently,
\begin{align}
     \frac{1}{\tan \theta \rbr{\vct v_1,u_{T} }} \geq (1+\alpha)^T \frac{1}{\tan \theta \rbr{\vct v_1,\vct u_0 }}.
\end{align}
Finally, we complete the proof by setting $T > \log_{1+\alpha} \rbr{1 + \rho}{\tan \theta \rbr{\vct v_1, \vct u_0}}$.
\end{proof}

% Second, we consider the case when $\tan \theta \rbr{\vct v_1,\vct u} < 1.$
{
Next, we present Lemma \ref{lem_second_phase}, which analyzes the second phase of the noisy tensor power method.
The second phase starts with $\tan\theta(\vct v_1,\vct u_0) < 1$, that is, the inner product of $\vct v_1$ and $\vct u_0$ is lower bounded by 1/2.
}
\begin{lem}
\label{lem_second_phase}
Let $\vct v_1$ be the principal eigenvector of a tensor $\mat T$ and let $\vct u_0$ be an arbitrary vector in $\mathbb R^d$ that satisfies $\tan\theta(\vct v_1,\vct u_0) < 1$.
Suppose at every iteration $t$ the noise satisfies
\begin{align}
   4\|\tilde{\vct\varepsilon}(\vct u_t)\| \leq \epsilon\rbr{\lambda_1 - \lambda_2} \,\,\, \text{and} \,\,\, 4\big| \langle\vct v_1,\tilde{\vct\varepsilon}(\vct u_t)\rangle\big| \leq \rbr{\lambda_1 - \lambda_2} \cos^2 \theta \rbr{\vct v_1, \vct u_0}
\end{align}
for some $\epsilon < 1.$
Then with high probability there exists $T = O\rbr{\frac{\lambda_1}{\lambda_1 - \lambda_2}\log (1/\epsilon)}$ such that after $T$ iteration we have
$ \tan \theta  \rbr{\vct v_1, \vct u_T} \leq \epsilon$.
\end{lem}

\begin{proof}
Define $\Delta := \frac{\lambda_1 - \lambda_2}{4}$ and $\mat X :=  \vct  v_1^\perp$.
We have the following chain of inequalities:
\begin{align}
\tan \theta  \rbr{\vct v_1, \mat T\rbr{\mat I,  \vct u, \vct u} + \tilde{\vct\varepsilon}(\vct u)}
&\leq \frac{ \nbr{ \mat X^T \rbr{\mat T\rbr{\mat I, \vct u, \vct u} + \tilde{\vct\varepsilon}(\vct u) }}}{ \nbr{ \vct v_1^T \rbr{\mat T\rbr{\mat I, \vct u,\vct u} + \tilde{\vct\varepsilon}(\vct u) } } } \\
& \leq \frac{\nbr{\mat X^T \mat T\rbr{\mat  I, \vct u, \vct u}} + \nbr{\mat V^T \tilde{\vct\varepsilon}(\vct u) }}{ \nbr{\vct v_1^T \mat T\rbr{\mat I, \vct u, \vct u}} - \nbr{\vct v_1^T \tilde{\vct\varepsilon}(\vct u)  }}\\
&\leq \frac{\lambda_2 \nbr{\mat X^T \vct u}^2 + \nbr{\tilde{\vct\varepsilon}(\vct u)  } }{ \lambda_1 \abr{\vct v_1^T \vct u }^2 - \big| \vct v_1^\top\tilde{\vct\varepsilon}(\vct u)\big|  } \\
& = \frac{\nbr{\mat X^T \vct u}^2}{\abr{ \vct v_1^T \vct u }^2 } \frac{\lambda_2}{\lambda_1 - \frac{ \abr{\vct v_1^\top\tilde{\vct\varepsilon}(\vct u) } }{\abr{\vct v_1^\top \vct u}^2}} + \frac{ \frac{ \nbr{ \tilde{\vct\varepsilon}(\vct u) } }{ \abr{\vct v_1^\top \vct u}^2 }  }{ \lambda_1 - \frac{ \big|\vct v_1^\top\tilde{\vct\varepsilon}(\vct u)\big| }{ \abr{\vct v_1^\top \vct u}^2 }   }\\
& \leq \tan^2 \theta  (\vct v_1, \vct u) \frac{\lambda_2}{\lambda_2 + 3\Delta} + \frac{\Delta \epsilon \rbr{1 + \tan^2 \theta \rbr{\vct v_1, \vct u}} }{\lambda_2 + 3\Delta}\\
& \leq \max \rbr{ \epsilon, \frac{\lambda_2 + \Delta \epsilon}{ \lambda_2 + 2 \Delta} \tan^2 \theta  \rbr{\vct v_1, \vct u}}\\
\label{eq_tangent_shrink}
& \leq \max \rbr{ \epsilon, \frac{\lambda_2 + \Delta \epsilon}{ \lambda_2 + 2 \Delta} \tan \theta  \rbr{\vct v_1, \vct u}}
\end{align}
The second step follows by triangle inequality. For $\vct u = \vct u_0 $, using the condition $\tan \rbr{\vct v_1, \vct u_0} < 1$ we obtain
\begin{align}
\tan \theta \rbr{\vct v_1, \vct u_1} \leq \max \rbr{ \epsilon, \frac{\lambda_2 + \Delta \epsilon}{ \lambda_2 + 2 \Delta} \tan^2 \theta  \rbr{\vct v_1, \vct u}} \leq \max \rbr{ \epsilon, \frac{\lambda_2 + \Delta \epsilon}{ \lambda_2 + 2 \Delta} \tan \theta  \rbr{\vct v_1, \vct u}}
\end{align}
Since $\frac{\lambda_2 + \Delta \epsilon}{ \lambda_2 + 2 \Delta} \leq \max \rbr{ \frac{\lambda_2}{\lambda_2 + \Delta}, \epsilon} \leq \rbr{\lambda_2/\lambda_1}^{1/4} < 1$, we have
\begin{align}
\tan \theta  \rbr{\vct v_1, \vct u_1} = \tan \theta  \rbr{\vct v_1, {\mat T} \rbr{\mat I, \vct u_0, \vct u_0}  + \tilde{\vct\varepsilon}(\vct u_t)  }  \leq \max \rbr{\epsilon, (\lambda_2 / \lambda_1)^{1/4} \tan \theta  \rbr{\vct v_1, \vct u_0}} < 1.
\end{align}
By induction,
\begin{align*}
\tan \theta \rbr{\vct v_1, \vct u_{t+1}} = \tan \theta  \rbr{\vct v_1, {\mat T} \rbr{\mat I, \vct u_t, \vct u_t}  + \tilde{\vct\varepsilon}(\vct u_t) } \leq \max \rbr{\epsilon, (\lambda_2 / \lambda_1)^{1/4} \tan \theta \rbr{\vct v_1, \vct u_t}} < 1.
\end{align*}
for every $t$.
Eq. \eqref{eq_tangent_shrink} then yields
\begin{align}
    \tan \theta \rbr{\vct v_1, \vct u_T} \leq \max \rbr{ \epsilon, \max{\epsilon, \rbr{\lambda_2/\lambda_1}^{L/4} \tan \theta \rbr{\vct v_1, \vct u_0}} }.
\end{align}
Consequently, after $T = \log \limits_{\rbr{\lambda_2/\lambda_1}^{-1/4}} (1/\epsilon)$ iterations we have $\tan \theta  \rbr{\vct v_1, \vct u_T} \leq \epsilon.$
\end{proof}

\begin{lem}
\label{lem_v_bound}
Suppose $\vct v_1$ is the principal eigenvector of a tensor $\mat T$ and let $\vct u_0 \in \mathbb{R}^n$. For some $\alpha, \rho > 0$ and $\epsilon <1,$ if at every step, the noise satisfies
\begin{align}
   \|\tilde{\vct\varepsilon}(\vct u_t)\| \leq  \epsilon\frac{\lambda_1-\lambda_2}{4}  \,\,\, \text{and} \,\,\, \big| \langle\vct v_1,\tilde{\vct\varepsilon}(\vct u_t)\rangle\big| \leq  \min \rbr{\frac{1}{4\frac{\max_{i \in [k]} \lambda_i}{\lambda_1} +2} \lambda_1, \frac{1- (1+\alpha)/2}{2\sqrt{2}(1+\alpha)} \lambda_1} \frac{1}{\tau^2 n},
\end{align}
then with high probability there exists an $T = O\rbr{ \log_{1+\alpha} \rbr{1+\rho} \tau \sqrt{n} + \frac{\lambda_1}{\lambda_1 - \lambda_2}\log (1/\epsilon)}$ such that
after $T$ iterations we have
$\nbr{\rbr{ \vct I - {\vct u}_T {\vct u}_T^T} \vct v_1} \leq \epsilon$.
\end{lem}

\begin{proof}
By Lemma 2.5 in \citeapp{noisy-tensor-power-method}, for any fixed orthonormal matrix $\mat V$ and a random vector $\vct u$, we have $\max_{i \in [K]} \tan \theta (\vct  v_i, \vct u) \leq \tau \sqrt{n}$ with all but $O(\tau^{-1}+e^{-\Omega(d)})$ probability. Using the fact that $\cos \theta \rbr{\vct v_1, \vct u_0} \geq 1/(1+\tan \theta \rbr{\vct v_1, \vct u_0}) \geq \frac{1}{\tau \sqrt{n}},$
the following bounds on the noise level imply the conditions in Lemma \ref{lem_first_phase}:
\begin{multline*}
\nbr{\mat V^T \tilde{\vct\varepsilon}(\vct u_t)} \leq \frac{1-(1+\alpha)/2}{2\sqrt{2}(1+\alpha) \tau \sqrt{n}}\,\,\, \text{and}\,\,\, \big| \langle\vct v_1,\tilde{\vct\varepsilon}(\vct u_t)\rangle\big|\\
 \leq  \min \rbr{\frac{1}{4\frac{\max_{i \in [k]} \lambda_i}{\lambda_1} +2} \lambda_1, \frac{1- (1+\alpha)/2}{2} \lambda_1} \frac{1}{\tau^2 n},\quad\forall t.
\label{eq_proof_lem_v_bound}
\end{multline*}
 Note that $\big| \langle\vct v_1,\tilde{\vct\varepsilon}(\vct u_t)\rangle\big| \leq  \frac{1- (1+\alpha)/2 }{2\sqrt{2}(1+\alpha)} \lambda_1 \frac{1}{\tau^2 n}$ implies the first bound
 in Eq. (\ref{eq_proof_lem_v_bound}).
In Lemma \ref{lem_second_phase}, we assume $\tan \theta \rbr{\vct v_1, \vct u_0}< 1$ and prove that for every
$\vct u_t,$ $\tan\theta  \rbr{\vct v_1, \vct u_t}< 1,$ which is equivalent to saying that at every step, $\cos \theta \rbr{\vct v_1, \vct u_t} > \frac{1}{\sqrt{2}}.$
By plugging the inequality into the second condition in Lemma \ref{lem_second_phase}, we have $\abr{ \inner{ \vct v_1}{ { \tilde{\vct \varepsilon}(\vct u_t)}}} \leq \frac{ \rbr{\lambda_1 - \lambda_2}}{8}$. %which is a relevantly loose condition when $n$ is large, which is the case in almost all the data.
The lemma then follows by the fact that $\nbr{\rbr{ \vct I - {\vct u}_T {\vct u_T}^T} \vct v_1} = \sin\theta  \rbr{\vct u_T, \vct v_1} \leq \tan \theta \rbr{\vct u_T, \vct v_1}\leq \epsilon$.
\end{proof}

\subsubsection{Deflation}

%We have proved that the Euclidean distance between the estimated largest eigenvector and the true largest eigenvector is bounded. In this section, we show that the error introduces from deflation is also bounded so that Lemma \ref{lem_v_bound} can also be applied on decomposition of $\hat{ \vct v}_2, \hat{\vct v}_3, \cdots \hat{\vct v}_K.$
{
In previous sections we have upper bounded the Euclidean distance between the estimated and the true principal eigenvector of an input tensor $\mat T$.
In this section, we show that error introduced from previous tensor power updates can also be bounded.
As a result, we obtain error bounds between the entire set of base vectors $\{\vct v_i\}_{i=1}^k$ and their estimation $\{\hat{\vct v}_i\}_{i=1}^k$.
}

\begin{lem}
\label{lem_deflation}
Let $\{\vct v_1, \vct v_2, \cdots, \vct v_k\}$ and $\cbr{\lambda_1, \lambda_2, \cdots, \lambda_k}$ be orthonormal eigenvectors and eigenvalues of an input tensor $T$.
Define $\lambda_{\max} := \max_{i \in [k]} \lambda_i$.
Suppose $\{\hat{\vct v}_i\}_{i=1}^k$ and $\{\hat{\lambda}_i\}_{i=1}^k$ are estimated eigenvector/eigenvalue pairs.
Fix $\epsilon \geq 0$ and any $t \in [k]$.
If
\begin{align}
\big|\hat{\lambda}_i - \lambda_i\big| \leq \lambda_i \epsilon/2,\,\,\, \text{and} \,\,\, \nbr{\hat{\vct u}_i - \vct u_i} \leq \epsilon
\end{align}
for all $i \in [t]$, then for any unit vector $\vct u$ the following holds:
\begin{align}
    \nbr{\sum \limits_{i=1}^t \sbr{\lambda \vct v_i^{\otimes 3} - \hat{\lambda}_i \hat{\vct v}_i^{\otimes 3}}\rbr{\mat I, \vct u, \vct u}}^2
\leq& 4 \rbr{2.5 \lambda_{\max} + (\lambda_{\max} + 1.5) \epsilon }^2\epsilon^2   +9 (1+\epsilon/2)^2 \lambda_{\max}^2 \epsilon^4 \\
&+ 8(1+ \epsilon/2)^2 \lambda_{\max}^2 \epsilon^2   \\
\leq& 50 \lambda_{\max}^2\epsilon^2.
\end{align}
\end{lem}
%for some constant $c$.
\begin{proof}
Following similar approaches in \citeapp{tensor-power-method}, Lemma B.5, we define $\hat{\vct v}^\perp = \hat{\vct v}_i - (\vct v_i^\top \hat{\vct v}_i)\vct v_i$ and $\mat D_i = \sbr{\lambda \vct v_i^{\otimes 3} - \hat{\lambda}_i \hat{\vct v}_i^{\otimes 3}}$.
$\mat D_i(\mat I, \vct u,\vct u)$ can then be written as the sum of scaled $\vct v_i$ and $\vct v_i^\top$ products as follows:
\begin{align}
\mat D_i \rbr{\mat I, \vct u, \vct u}=& \lambda_i (\vct u^\top \vct v_i)^2 \vct v_i - \hat{\lambda}_i (\vct u^\top \hat{\vct v}_i)^2 \hat{\vct v}_i\\
 =& \lambda_i (\vct u^\top \vct v_i)^2 \vct v_i - \hat{\lambda}_i (\vct u^\top \rbr{\hat{\vct v}_i^\perp + (\vct v_i^\top \hat{\vct v}_i)\vct v_i})^2 \rbr{\hat{\vct v}^\perp + (\vct v_i^\top \hat{\vct v}_i)\vct v_i}\\
 =& \rbr{\rbr{\lambda_i - \hat{\lambda}_i(\vct v_i^\top \hat{\vct v}_i)^3 }(\vct u^\top \vct v_i)^2 - 2\hat{\lambda}_i (\vct u^\top \hat{\vct v}_i^\perp)(\vct v_i^\top \hat{\vct v}_i)^2 (\vct u^\top \vct v_i) - \hat{\lambda}_i (\vct v_i^\top \hat{\vct v}_i )(\vct u^\top \hat{\vct v}^\perp)  } \vct v_i \nonumber \\
&- \hat{\lambda}_i \nbr{\hat{\vct v}_i^\perp} \rbr{ (\vct u^\top \vct v_i)(\vct v_i^\top \hat{\vct v}_i) + \vct u^\top \hat{\vct v}_i^\perp }\rbr{\hat{\vct v}_i^\perp /\nbr{\hat{\vct v}_i^\perp}}
\end{align}
Suppose $A_i$ and $B_i$ are coefficients of $\vct v_i$ and $\rbr{\hat{\vct v}_i^\perp /\nbr{\hat{\vct v}_i^\perp}}$, respectively.
The summation of $\mat D_i$ can be bounded as
\begin{align*}
 \nbr{\sum \limits_{i=1}^t \mat D_i \rbr{\mat I, \vct u, \vct u}}^2 =& \nbr{\sum \limits_{i=1}^t A_i \vct v_i - \sum \limits_{i=1}^t B_i \rbr{\hat{\vct v}_i^\perp /\nbr{\hat{\vct v}_i^\perp}}}_2^2\\
 \leq& 2 \nbr{\sum \limits_{i=1}^t A_i \vct v_i}^2 + 2 \nbr{\sum \limits_{i=1}^t B_i \rbr{\hat{\vct v}_i^\perp /\nbr{\hat{\vct v}_i^\perp}}}^2\\
 \leq& \sum \limits_{i=1}^t A_i^2 + 2\rbr{\sum \limits_{i=1}^t \abr{B_i}}^2
\end{align*}
We then try to upper bound $\abr{A_i}$.
\begin{align}
\abr{A_i} \leq& \abr{\rbr{\lambda_i - \hat{\lambda}_i(\vct v_i^\top \hat{\vct v}_i)^3 }(\vct u^\top \vct v_i)^2 - 2\hat{\lambda}_i (\vct u^\top \hat{\vct v}_i^\perp)(\vct v_i^\top \hat{\vct v}_i)^2 (\vct u^\top \vct v_i) - \hat{\lambda}_i (\vct v_i^\top \hat{\vct v}_i )(\vct u^\top \hat{\vct v}^\perp) }\\
     \leq&  \rbr{\lambda_i \abr{1- (\vct v_i^\top \hat{\vct v}_i)^3 } + \abr{ \lambda_i -  \hat{\lambda}_i } (\vct v_i^\top \hat{\vct v}_i)^3 }(\vct u^\top \vct v_i)^2 + 2\rbr{\lambda_i +\abr{\lambda_i - \hat{\lambda}_i } } \nbr{ \hat{\vct v}_i - \vct v_i}\abr{\vct u^\top \vct v_i} \nonumber\\
     &+\rbr{\lambda_i +\abr{\lambda_i - \hat{\lambda}_i } } \nbr{\hat{\vct v}_i - \vct v_i}^2 \\
      \leq& \rbr{1.5 \nbr{\vct v_i - \hat{\vct v}_i}^2 + \abr{ \lambda_i -  \hat{\lambda}_i } + 2 \rbr{\lambda_i +\abr{\lambda_i - \hat{\lambda}_i }} \nbr{\vct v_i - \hat{\vct v}_i}} \abr{\vct u^\top \vct v_i} \nonumber\\
      &+ \rbr{\lambda_i +\abr{\lambda_i - \hat{\lambda}_i } } \nbr{\hat{\vct v}_i - \vct v_i}^2\\
  \leq & \rbr{2.5 \lambda_i + (\lambda_i + 1.5) \epsilon }\epsilon \abr{\vct u^\top \vct v_i} + (1+\epsilon/2) \lambda_i \epsilon^2
\end{align}
Next, we bound $\abr{B_i}$ in a similar manner.
\begin{align}
\abr{B_i} =& \abr{\hat{\lambda}_i \nbr{\hat{\vct v}_i^\perp} \rbr{ (\vct u^\top \vct v_i)(\vct v_i^\top \hat{\vct v}_i) + \vct u^\top \hat{\vct v}_i^\perp }}\\
          \leq& 2\rbr{\lambda_i + \abr{\lambda_i - \hat{\lambda}_i}} \nbr{\hat{\vct v}_i^\perp}\rbr{ (\vct u^\top \vct v_i)^2 + \nbr{\hat{\vct v}_i^\perp}^2 }\\
          \leq& 2(1+ \epsilon/2) \lambda_i\epsilon (\vct u^\top \vct v_i)^2 + 2(1+ \epsilon/2) \lambda_i \epsilon^3
\end{align}
\end{proof}
Combining everything together we have
\begin{align}
\nbr{\sum \limits_{i=1}^t \mat D_i \rbr{\mat I, \vct u,\vct u}}^2\leq& 2 \sum \limits_{i=1}^t A_i^2 + 2\rbr{\sum \limits_{i=1}^t \abr{B_i}}^2\\
\leq& \sum \limits_{i=1}^t 4 \rbr{5 \lambda_i + (\lambda_i + 1.5)  }^2\epsilon^2 \abr{\vct u^\top \vct v_i}^2 +4 (1+\epsilon/2)^2 \lambda_i^2 \epsilon^4 \nonumber\\
&+2\rbr{\sum \limits_{i=1}^t 2(1+ \epsilon/2) \lambda_i \epsilon (\vct u^\top \vct v_i)^2 + 2(1+ \epsilon/2) \lambda_i \epsilon^3}^2\\
\leq&4 \rbr{2.5 \lambda_{\max} + (\lambda_{\max} + 1.5) \epsilon  }^2\epsilon^2  \sum \limits_{i=1}^t \abr{\vct u^\top \vct v_i}^2 +4 (1+\epsilon/2)^2 \lambda_{\max}^2 \epsilon^4\nonumber\\
&+ 2\rbr{2(1+ \epsilon/2) \lambda_{\max} \epsilon \sum \limits_{i=1}^t (\vct u^\top \vct v_i)^2 + 2(1+ \epsilon/2) \lambda_{\max}\epsilon^3}^2\\
\leq& 4 \rbr{2.5 \lambda_{\max} + (\lambda_{\max} + 1.5) \epsilon }^2\epsilon^2   +9 (1+\epsilon/2)^2 \lambda_{\max}^2 \epsilon^4 + 8(1+ \epsilon/2)^2 \lambda_{\max}^2 \epsilon^2.
\end{align}

\subsubsection{Main Theorem}

{
In this section we present and prove the main theorem that bounds the reconstruction error of fast robust tensor power method
under appropriate settings of the hash length $b$ and number of independent hashes $B$.
The theorem presented below is a more detailed version of Theorem \ref{thm_tensor_power} presented in Section \ref{subsec:analysis_tensor_power_method}.
}

\begin{thm} Let $\bar{\mat T} = \mat T + \mat E \in \mathrm{R}^{n \times n \times n}$, where
$\mat T=\sum_{i=1}^k{\lambda_i\vct v_i^{\otimes 3}}$ and $\{\vct v_i\}_{i=1}^k$ is an orthonormal basis.
Suppose $(\hat{\vct v}_1, \hat{\lambda}_1), (\hat{\vct v}_1, \hat{\lambda}_1), \cdots (\hat{\vct v}_k, \hat{\lambda}_k)$ is the sequence of estimated eigenvector/eigenvalue pairs obtained using the fast robust tensor power method. Assume $\nbr{\mat E} = \epsilon.$ There exists constant $C_1, C_2, C_3, \alpha, \rho, \tau \geq 0$ such that the following holds:
if
\begin{align}
 \epsilon \leq C_1 \frac{1}{n\lambda_{\max}}, \,\,\, \text{and} \,\,\, T = C_2\rbr{ \log_{1+\alpha} \rbr{1+\rho} \tau \sqrt{n} + \frac{\lambda_1}{\lambda_1 - \lambda_2}\log (1/\epsilon)},
 \end{align}
 and
 \begin{align}
  \sqrt{\frac{\ln (L/\log_2(k/\eta))}{\ln(k)}} \cdot \rbr{
 1 - \frac{\ln \rbr{\ln L/\log_2(k/\eta)} + C_3}{4 \ln \rbr{L/\log_2(k/\eta)}}- \sqrt{\frac{\ln (8)}{\ln(L/\log_2(k/\eta))}} }
   \geq 1.02 \rbr{1 + \sqrt{\frac{\ln(4)}{\ln(k)}}}.
\end{align}
Suppose the tensor sketch randomness is independent among all tensor product evaluations.
If $B=\Omega(\log(n/\tau))$ and the hash length $b$ is set to
\begin{align}
b \geq  \left\{\frac{\nbr{\mat T}_F^2 \tau^4 n^2}{\min \rbr{\frac{1}{4\max_{i \in [k]} (\lambda_i/\lambda_1) +2} \lambda_1, \frac{1- (1+\alpha)/2}{2\sqrt{2}(1+\alpha)} \lambda_1}^2},
\frac{16 \epsilon^{-2}\|\mat T\|_F^2}{\min_{i\in[k]}{(\lambda_i-\lambda_{i-1})^2}}  , \epsilon^{-2}\nbr{\mat T}_F^2 \right\}
\end{align}
with probability at least $1 - (\eta + \tau^{-1} + e^{-n} ) $, there exists a permutation $\pi$ on $k$ such that
\begin{align}
   \nbr{\vct v_{\pi(j)} - \hat{\vct v}_i} \leq \epsilon, \,\,\, \abr{\lambda_{\pi(j)}-\hat{\lambda}_j} \leq \frac{\lambda_{\pi(j)}\epsilon}{2}, \,\,\, \text{and} \,\,\,
   \nbr{\mat T - \sum \limits_{j= 1}^k \hat{\lambda_j}\hat{\vct v}_j^{\otimes 3}} \leq c \epsilon,
\end{align}
for some absolute constant $c$.
\label{thm_refine_tensor_power}
\end{thm}

\begin{proof}
We prove that at the end of each iteration $i \in [k]$, the following conditions hold

\begin{itemize}
\item 1. For all $ j \leq i, \abr{\vct v_{\pi(j)} - \hat{\vct v}_j} \leq \epsilon\,\,\, \text{and} \,\,\, \abr{\lambda_{\pi(j)}-\hat{\lambda}_j} \leq \frac{\lambda_i\epsilon}{2}$
\item 2. The tensor error satisfies
\begin{align}
     \nbr{ \sbr{ \rbr{\tilde{\mat T} - \sum \limits_{j \leq i} \hat{\lambda}_j \hat{\vct v}_j^{\otimes 3} } - \sum \limits_{j \geq i+1} \lambda_{\pi(j)} v_{\pi(j)}^{\otimes 3}} \rbr{\mat I, \vct u, \vct u} } \leq 56 \epsilon
\end{align}
 \end{itemize}
 First, we check the case when $i = 0.$  For the tensor error, we have
 \begin{align}
     \nbr{ \sbr{ \tilde{\mat T}  - \sum \limits_{j = 1}^K \lambda_{\pi(j)} v_{\pi(j)}^{\otimes 3}} \rbr{\mat I, \vct u, \vct u} } = \nbr{\vct \varepsilon(\vct u) } \leq  \nbr{\vct \varepsilon_{2,T}(\vct u) } + \nbr{\mat E \rbr{\mat I, \vct u,\vct u}} \leq \epsilon + \epsilon = 2\epsilon.
\end{align}
The last inequality follows Theorem \ref{thm_tensor_error} with the condition for $b$. Next, Using Lemma \ref{lem_v_bound}, we have that
\begin{align}
\nbr{\vct v_{\pi(1)} - \hat{\vct v}_1} \leq \epsilon.
\end{align}
In addition, conditions for hash length $b$ and Theorem \ref{thm_tensor_error} yield
\begin{align}
\abr{\lambda_{\pi(1)}-\hat{\lambda}_1} \leq \nbr{\vct \varepsilon_{1,T}(\vct v_1)} + \nbr{\mat{T}(\hat{\vct v_1} - \vct v_1, \hat{\vct v_1} - \vct u, \hat{\vct v_1} - \vct v_1)} \leq \epsilon \frac{\lambda_i - \lambda_{i-1}}{4} + \epsilon^3 \nbr{\mat T}_F \leq \frac{\epsilon \lambda_i}{2}
\end{align}
Thus, we have proved that for $i = 1$ both conditions hold.
Assume the conditions hold up to $i=t-1$ by induction.
For the $t$th iteration, the following holds:
\begin{multline*}
     \nbr{ \sbr{ \rbr{\tilde{\mat T} - \sum \limits_{j \leq t} \hat{\lambda}_j \hat{\vct v}_j^{\otimes 3} } - \sum \limits_{j \geq t+1} \lambda_{\pi(j)} v_{\pi(j)}^{\otimes 3}} \rbr{\mat I, \vct u, \vct u} }\\
     \leq  \nbr{ \sbr{ \tilde{\mat T}  - \sum \limits_{j = 1}^K \lambda_{\pi(j)} v_{\pi(j)}^{\otimes 3}} \rbr{\mat I, \vct u, \vct u} } + \nbr{\sum \limits_{j = 1}^t \hat{\lambda}_j \hat{\vct v}_j^{\otimes 3} -\lambda_{\pi(j)} v_{\pi(j)}^{\otimes 3} }
    \leq \epsilon + \sqrt{50} \lambda_{\max} \epsilon.
\end{multline*}
For the last inequality we apply Lemma \ref{lem_deflation}. Since the condition is satisfied, Lemma \ref{lem_v_bound} yields
\begin{align}
\nbr{\vct v_{\pi(t+1)} - \hat{\vct v}_{t+1}} \leq \epsilon.
\end{align}
Finally, conditions for hash length $b$ and Theorem \ref{thm_tensor_error} yield
\begin{multline}
\abr{\lambda_{\pi(t+1)}-\hat{\lambda}_{t+1}} \leq \nbr{\vct \varepsilon_{1,T}(\vct v_1)} + \nbr{\mat{T}(\hat{\vct v_{t}} - \vct v_1, \hat{\vct v_1} - \vct u, \hat{\vct v_1} - \vct v_1)} \\
\leq \epsilon \frac{\lambda_i - \lambda_{i-1}}{4} + \epsilon^3 \nbr{\mat T}_F \leq \frac{\epsilon \lambda_i}{2}
\end{multline}

\end{proof}

\section{Summary of notations for matrix/vector products}\label{appsec:notation}

We assume vectors $\vct a,\vct b\in\mathbb C^n$ are indexed starting from 0;
that is, $\vct a = (a_0,a_1,\cdots,a_{n-1})$ and $\vct b=(b_0,b_1,\cdots,b_{n-1})$.
Matrices $\mat A,\mat B$ and tensors $\mat T$ are still indexed starting from 1.

\paragraph{Element-wise product} For $\vct a,\vct b\in\mathbb C^n$, the element-wise product (Hadamard product) $\vct a\circ\vct b\in\mathbb R^n$ is defined as
\begin{equation}
\vct a\circ\vct b = (a_0b_0,a_1b_1,\cdots, a_{n-1}b_{n-1}).
\end{equation}

\paragraph{Convolution} For $\vct a,\vct b\in\mathbb C^n$, their convolution $\vct a*\vct b\in\mathbb C^n$ is defined as
\begin{equation}
\vct a*\vct b = \left(\sum_{(i+j)\mod n=0}{a_ib_j},\sum_{(i+j)\mod n=1}{a_ib_j}, \cdots, \sum_{(i+j)\mod n=n-1}{a_ib_j}\right).
\end{equation}

\paragraph{Inner product} For $\vct a,\vct b\in\mathbb C^n$, their inner product is defined as
\begin{equation}
\langle\vct a,\vct b\rangle = \sum_{i=1}^n{a_i\overline{b_i}},
\end{equation}
where $\overline{b_i}$ denotes the complex conjugate of $b_i$.
For tensors $\mat A,\mat B\in\mathbb C^{n\times n\times n}$, their inner product is defined similarly as
\begin{equation}
\langle\mat A,\mat B\rangle = \sum_{i,j,k=1}^n{\mat A_{i,j,k}\overline{\mat B}_{i,j,k}}.
\end{equation}

\paragraph{Tensor product} For $\vct a,\vct b\in\mathbb C^n$, the tensor product $\vct a\otimes\vct b$ can be either an $n\times n$ matrix or a vector of length $n^2$.
For the former case, we have
\begin{equation}
\vct a\otimes \vct b = \left[\begin{array}{cccc}
a_0b_0& a_0b_1& \cdots& a_0b_{n-1}\\
a_1b_0& a_1b_1& \cdots& a_1b_{n-1}\\
\vdots& \vdots& \ddots& \vdots\\
a_{n-1}b_0& a_{n-1}b_1& \cdots& a_{n-1}b_{n-1}\\\end{array}\right].
\end{equation}
If $\vct a\otimes\vct b$ is a vector, it is defined as the expansion of the output matrix. That is,
\begin{equation}
\vct a\otimes\vct b = (a_0b_0,a_0b_1,\cdots,a_0b_{n-1},a_1b_0,a_1b_1,\cdots,a_{n-1}b_{n-1}).
\end{equation}

Suppose $\mat T$ is an $n\times n\times n$ tensor and matrices $\mat A\in\mathbb R^{n\times m_1}$, $\mat B\in\mathbb R^{n\times m_2}$ and $\mat C\in\mathbb R^{n\times m_3}$.
The tensor product $\mat T(\mat A,\mat B,\mat C)$ is an $m_1\times m_2\times m_3$ tensor defined by
\begin{equation}
\left[\mat T(\mat A,\mat B,\mat C)\right]_{i,j,k} = \sum_{i',j',k'=1}^n{\mat T_{i',j',k'}\mat A_{i',i}\mat B_{j',j}\mat C_{k',k}}.
\end{equation}

\paragraph{Khatri-Rao product} For $\mat A,\mat B\in\mathbb C^{n\times m}$, their Khatri-Rao product $\mat A\odot\mat B\in\mathbb C^{n^2\times m}$ is defined as
\begin{equation}
\mat A\odot\mat B = (\mat A_{(1)}\otimes\mat B_{(1)}, \mat A_{(2)}\otimes\mat B_{(2)}, \cdots,\mat A_{(m)}\otimes\mat B_{(m)}),
\end{equation}
where $\mat A_{(i)}$ and $\mat B_{(i)}$ denote the $i$th rows of $\mat A$ and $\mat B$.

\paragraph{Mode expansion}
For a tensor $\mat T$ of dimension $n\times n\times n$, its first mode expansion $\mat T_{(1)}\in\mathbb R^{n\times n}$ is defined as
\begin{equation}
\mat T_{(1)} = \left[\begin{array}{ccccccc}
\mat T_{1,1,1}& \mat T_{1,1,2}& \cdots& \mat T_{1,1,n}& \mat T_{1,2,1}& \cdots& \mat T_{1,n,n}\\
\mat T_{2,1,1}& \mat T_{2,1,2}& \cdots& \mat T_{2,1,n}& \mat T_{2,2,1}& \cdots& \mat T_{2,n,n}\\
\vdots& \vdots& \vdots& \vdots& \vdots& \vdots& \vdots\\
\mat T_{n,1,1}& \mat T_{n,1,2}& \cdots& \mat T_{n,1,n}& \mat T_{n,2,1}& \cdots& \mat T_{n,n,n}\\
\end{array}\right].
\end{equation}
The mode expansions $\mat T_{(2)}$ and $\mat T_{(3)}$ can be similarly defined.

\bibliographystyleapp{IEEE}
\bibliographyapp{fftlda}

\end{appendices}

\end{document}